\documentclass{article}

\usepackage{microtype}
\usepackage{graphicx}
\usepackage{subfigure}
\usepackage{booktabs} 

\usepackage{hyperref}

\usepackage[accepted]{icml2022}

\usepackage{amsmath}
\usepackage{amssymb}
\usepackage{mathtools}
\usepackage{amsthm}
\usepackage{float}
\usepackage{mathrsfs}
\usepackage{enumerate}
\usepackage{bm}
\usepackage{enumitem}
\usepackage[capitalize,noabbrev, nameinlink]{cleveref}

\theoremstyle{plain}
\newtheorem{theorem}{Theorem}[section]
\newtheorem{proposition}[theorem]{Proposition}

\newtheorem{corollary}[theorem]{Corollary}
\theoremstyle{definition}
\newtheorem{definition}[theorem]{Definition}

\theoremstyle{remark}

\usepackage{tikz,pgfplots}

\newenvironment{talign}
 {\let\displaystyle\textstyle\align}
 {\endalign}
\newenvironment{talign*}
 {\let\displaystyle\textstyle\csname align*\endcsname}
 {\endalign}

\newcommand{\lelc}{\le_{\mathrm{lc}}}

\newcommand{\gelc}{\ge_{\mathrm{lc}}}

\usepackage[textsize=tiny]{todonotes}
\usepackage{multirow, makecell}
\usepackage{stfloats} 

\icmltitlerunning{Why the Rich Get Richer? On the Balancedness of Random Partition Models}
\pgfplotsset{compat=1.17}
\begin{document}

\twocolumn[
\icmltitle{Why the Rich Get Richer? \\
On the Balancedness of Random Partition Models}




\icmlsetsymbol{equal}{*}

\begin{icmlauthorlist}
\icmlauthor{Changwoo J. Lee}{tamu}
\icmlauthor{Huiyan Sang}{tamu}
\end{icmlauthorlist}

\icmlaffiliation{tamu}{Department of Statistics, Texas A\&M University, Texas, USA}

\icmlcorrespondingauthor{Changwoo J. Lee}{c.lee@stat.tamu.edu}

\icmlkeywords{Bayesian nonparametrics, Random partition models, Clustering, Balancedness, Entity resolution}

\vskip 0.3in
]



\printAffiliationsAndNotice{}  

\begin{abstract}
Random partition models are widely used in Bayesian methods for various clustering tasks, such as mixture models, topic models, and community detection problems. 
While the number of clusters induced by random partition models has been studied extensively, another important model property regarding the balancedness of partition has been largely neglected.
We formulate a framework to define and theoretically study the balancedness of exchangeable random partition models, by analyzing how a model assigns probabilities to partitions with different levels of balancedness. 
We demonstrate that the ``rich-get-richer'' characteristic of many existing popular random partition models is an inevitable consequence of two common assumptions: product-form exchangeability and projectivity.
We propose a principled way to compare the balancedness of random partition models, which gives a better understanding of what model works better and what doesn't for different applications.
We also introduce the ``rich-get-poorer'' random partition models and illustrate their application to entity resolution tasks.
\end{abstract}

\section{Introduction}
\label{sec:intro}

Bayesian formulation of clustering problems has gained considerable attention, especially when the number of clusters is unknown. By assigning a prior distribution on the partition space, Bayesian clustering methods give not only flexibility in the number of clusters but also an uncertainty quantification of latent clustering structure. 
Popular examples include random partition models induced by discrete random probability measures, such as Dirichlet process \citep[DP;][]{ferguson1973bayesian}, Pitman-Yor process \citep[PYP;][]{perman1992size,pitman1997two}, and the mixture of finite mixtures \citep[MFM;][]{miller2018mixture}, to name a few. 
These random partition models reflect the prior belief on the clusters' characteristics, such as how the number of clusters $K^+_n$ grows as the number of datapoints $n$ increases, and how datapoints are distributed among clusters. 
For example, it is well known that $K^+_n\asymp\log n$ under the Chinese restaurant process (CRP), the distribution of partition induced by DP. 

While the behavior of the number of clusters $K^+_n$ under various random partition models has been studied extensively \citep{gnedin2006exchangeable,lijoi2007controlling,favaro2009bayesian,de2015gibbs,camerlenghi2018bayesian,lijoi2020pitman,fruhwirth2021generalized}, the  \textit{balancedness} of random partition (i.e. probabilities of partitions with different levels of balancedness) has been largely uninvestigated. 
Nevertheless, there has been awareness of the crucial effect of the balancedness of random partitions on cluster inference. The most prominent example is the CRP prior that is known to yield the posterior estimate with few big and many small extra clusters \citep{petrone1997note}, intuitively due to the prediction rule of CRP that assumes a new datapoint falls into an existing cluster with probability proportional to its size. 
This ``rich-get-richer'' characteristic of preferring unbalanced partitions is shared across many other random partition models, such as those induced by PYP, MFM, and Dirichlet-multinomial allocation~\citep{green2001modelling}, yet there is a less theoretical understanding of when and why this property arises. 

There are many important questions surrounding the balancedness of random partition models that are not well studied in the literature. 
First, how to \textit{define} a rigorous theoretical framework that is applicable to systematically studying the balancedness properties for a broad class of random partition models? 
The answer to this question would be crucial to understand the root causes of the ``rich-get-richer'' property: the product-form exchangeability and projectivity assumptions which we shall demonstrate later. 
Also, it provides deeper insight into the suitable application scenario of each model. The success of latent Dirichlet allocation \citep[LDA;][]{blei2003latent} in topic modeling is partially attributed to the fact that per-document topic proportions are inherently unbalanced, and hence unbalanced random partition models are naturally suitable for topic modeling. 

Second, how to \textit{compare} and \textit{quantify} the balancedness of different random partition models?  
The answer to this question helps to understand the cost and benefits of choosing one model over the other. 
For example, many practitioners choose DP-based models since it does not require the number of components to be fixed in advance, and indeed many methods incorporating DP emphasize this as a key feature \citep{teh2006hierarchical, kemp2006learning}. 
However, this benefit comes with a cost: compared to the Dirichlet-multinomial allocation model, the random partition induced by DP is more unbalanced, which explains why the DP mixture model typically yields too many tiny clusters compared to the finite mixture or MFM  \citep{green2001modelling,miller2018mixture}. 
Also, community detection results based on the Bayesian stochastic blockmodel~\citep{nowicki2001estimation,geng2019probabilistic} are reported to be sensitive to the choice of the latent random partition model~\citep{legramanti2022extended}, and this may be partly explained by the balancedness of random partition prior; see \cref{subsec:3.3} for the detailed discussion.  

Finally, beyond the sea of models that implicitly favor unbalanced partitions, are there flexible models that favor balancedness or have no preference on the balancedness? 
There have been some efforts to follow a sequential construction of random partition  and redesign prediction rules to change the ``rich-get-richer'' property \citep{jensen2008bayesian,wallach2010alternative,lu2018reducing,poux2021powered}, but these sequential construction methods (i.e. species sampling models) often no longer guarantee exchangeability of the resulting random partition model~\citep{lee2013defining}, a crucial  assumption that the distribution of random partition is invariant of the permutation of indices. 
Indeed, there are many application scenarios for which such models are more suitable. 
In this paper, we illustrate with an entity resolution (ER) task which aims to find clusters of records that refer to the same individual in a noisy database where we do not expect a small number of clusters dominates the whole. 

\textbf{Scope and Contributions}. First, we formulate the balancedness of exchangeable random partitions and study the related properties, which is crucial to check or build the model having the desired balancedness property.
Especially, we focus on Gibbs partition models \citep{gnedin2006exchangeable}, a very general class of exchangeable random partition models.
Second, we provide the complete family of balance-neutral random partitions that is both exchangeable and projective, which can serve as a noninformative prior in terms of balancedness. 
Third, we define \textit{B-sequence} to compare the strength of balancedness between Gibbs partitions and offer its intuitive explanation. 
Finally, we provide an in-depth example of a balance-seeking random partition model and demonstrate its application to the ER task.

\section{Backgrounds}

\subsection{Notations}

Let $[n]=\{1, \ldots, n\}$ denote the index set of datapoints. A partition of $[n]$ is a collection of unordered nonempty subsets (clusters) $\Pi_{n}\!=\!\left\{S_{1}, \ldots, S_{k}\right\}$ with $\cup_j S_{j}=[n]$ and $S_{j} \cap S_{\ell}\!=\!\emptyset$ for $j \neq \ell$. 
Let $\mathcal{P}_{[n]}^{k}$ denote the set of partitions of $[n]$ into $k$ blocks, and $\mathcal{P}_{[n]}\!=\!\cup_{k=1}^{n} \mathcal{P}_{[n]}^{k}$ be the set of all partitions~of~$[n]$.   
We use notations $n_{j}=\left|S_{j}\right|$ and $\boldsymbol{n}=\left(n_{1}, \ldots, n_{k}\right)$ to denote the cluster sizes and $k=\left|\Pi_n\right|$ to denote the number of clusters in partition $\Pi_n$. 
Let $z_i\in\{1,\dots,k\}$ be a cluster membership of $i$-th datapoint and $\mathbf{z} = (z_1,\ldots,z_n)$ be a cluster membership vector.
Further, let $\mathcal{I}^k_n$ be a collection of integer partition of $n$ into $k$ parts, a nonincreasing $k-$tuple of positive integers $\bm{n} = (n_1,\ldots, n_k)$ whose sum is $n$. 
Similarly, define $\mathcal{I}_n = \cup_{k=1}^n\mathcal{I}_n^k$. 
For a single $\bm{n}\in \mathcal{I}^k_n$, there are $n!/(\prod_{i=1}^n (i!)^{m_i}m_i!)$ number of partitions in $\mathcal{P}_{[n]}^k$ whose nonincreasing cluster sizes correspond to $\bm{n}$, where $m_i=\sum_{j=1}^k1(n_j = i)$ is the number of clusters of size $i$.

\subsection{Comparison of Balancedness Between Partitions}

Next we consider how to compare the balancedness of two \textit{fixed} partitions with cluster sizes $\bm{n}=(n_1,\ldots,n_k)$ and $\bm{n}'=(n_1',\ldots,n_k')$. 
We restrict our attention to the case when both partitions have $k$ clusters.
A diversity index \citep{hill1973diversity} is a measure to quantify the evenness of species distribution widely used in the ecology literature. To compare the balancedness of two partitions using the diversity index, we treat the proportion of cluster sizes $\{n_j/n\}_{j=1}^k$ as species distribution.
Popular examples are the Shannon index (with base $e$) and the Gini-Simpson index:
\begin{talign}
    H(\bm{n}) &:=  -\sum_{j=1}^k(n_j/n)\log (n_j/n) \quad \text{(Shannon)}\\
    G(\bm{n}) &:= 1-\sum_{j=1}^k(n_j/n)^2
    \quad \text{(Gini-Simpson)}
\end{talign}
Other examples include the R\'enyi index and the Hill numbers \citep{daly2018ecological}.\footnote{We refrain from using the term \textit{entropy} to avoid confusion since we don't consider $\bm{n}$ as a random variable in this context.} 
A higher diversity index indicates a partition is more balanced (has clusters with similar sizes), and the index is maximized when all $n_j/n$ are equal. 

However, different diversity indices do not always give a consistent ordering \citep{hurlbert1971nonconcept}, and they may not provide meaningful comparisons for some pairs of partitions. For instance, consider partitions of $[10]$ with cluster sizes $\bm{n} = (6,2,2)$ and $\bm{n}' = (5,4,1)$. 
It is not clear which one is more balanced than the other because both can be obtained by moving one data from a larger cluster of the partition $(6,3,1)$ to a smaller cluster, and diversity indices are inconsistent: $H(\bm{n}) \approx 0.95 > 0.94 \approx H(\bm{n}')$ but $G(\bm{n}) \approx 0.56 < 0.58 \approx G(\bm{n}')$. 
Instead, without referring to indices, \citet{patil1982diversity} proposed  intrinsic diversity ordering, a partial order relation to compare diversity between two species distribution.
When it comes to the discrete space $\mathcal{I}^k_n$, intrinsic diversity ordering now corresponds to the reverse dominance order.  

\begin{definition}[Reverse dominance order] Let $\bm{n} = (n_1,\ldots,n_k)$ and $\bm{n}' = (n_1',\ldots,n_k')$ be two integer partitions of $n$ into $k$ parts. By \citet{brylawski1973lattice}, the following statements are equivalent definitions of partial order $\prec$ on $\mathcal{I}_{n}^k$ where $\bm{n}\prec \bm{n}'$ represents $\bm{n}'$ is more balanced than $\bm{n}$:
\begin{enumerate}[label=(\alph*)]
    \item (reverse dominance ordering) Write $\bm{n} \prec \bm{n}'$ if $\sum_{j=1}^J n_j \ge \sum_{j=1}^J n_j'$ for $J=1,\dots,k$ and $\bm{n}\neq \bm{n}'$.
    \item (one-step downshift) Write $\bm{n} \prec \bm{n}'$ if $\bm{n}$ leads to $\bm{n}'$ by a finite sequence of one-step downshifts:
    $
    (n_1, \dots, n_k) \mapsto (n_1,\dots,n_u-1,\dots, n_v + 1, \dots, n_k)
    $
    provided that $n_u-1 \ge n_v+1$, where $1\le u <v\le k$. 
    \item (covering relation) $\prec$ is induced by the following covering relation:  $\bm{n'}$ covers $\bm{n}$ (i.e. no $\bm{n}''$ exist s.t. $\bm{n}\prec \bm{n}''\prec \bm{n}'$) if $n_u'=n_u-1, n_v' = n_v+1$ for some $u<v$, $n_j=n_j'$ for all $j\in \{1,\ldots,k\}\backslash\{u,v\}$, and either $(*)$ $v = u+1 $ or $(**)$ $n_u' = n_v'$, or both.

\end{enumerate}
\label{def:order}
\end{definition}
See \cref{fig:order} for an example of $\prec$. 
Although \cref{def:order} can be extended to the set $\mathcal{I}_n=\cup_{k=1}^n \mathcal{I}_n^k$, 
we shall only focus on the ordering between integer partitions with the same $k$ to study the balancedness of random partition models.

\subsection{Gibbs Partitions}
\label{subsec:2.3}
Random partition of $[n]$ is a discrete probability distribution on the space of partitions $\mathcal{P}_{[n]}$. There are two common assumptions on random partition models: \textit{finite exchangeability} and \textit{projectivity}. A random partition $\Pi_n$ on $[n]$ is called \textit{finitely exchangeable} if probabilities remain unchanged under any permutation of the indices of datapoints, which is a natural assumption unless one wants to take account of data dependence structures.
Next, a sequence of random partitions $(\Pi_n)_{n=1}^\infty$ satisfies \textit{projectivity} if the restriction of $\Pi_n$ to $[m]$ is $\Pi_m$ for any $m < n$, which is also called the addition rule, self-consistency, or marginal invariance property. We call $(\Pi_n)$ \textit{infinitely exchangeable} if $\Pi_n$ is finitely exchangeable for each $n$ and satisfies projectivity.

\begin{figure}[t]
\centering
\begin{tikzpicture}[scale=9, every node/.style={scale=0.8}]
 \foreach \y in {0.6,0.7,0.8,0.9,1, 1.1}{
  \draw [help lines, color=gray!50, dashed] (-0.25,\y) -- (0.25,\y);
  }
\foreach \y in {0.6,0.7,0.8,0.9,1.0, 1.1}
 \node[anchor=east] at (-0.25,\y) {\y};
\draw[->,thick] (-0.25,0.6)--(-0.25,1.15) node[above]{$H$ (Shannon)};
  \node (433) at (0,1.089) {\small $(4,3,3)$};
  \node (442) at (0.15,1.055) {\small $(4,4,2)$};
  \node (532) at (-0.15,1.03) {\small $(5,3,2)$};
  \node (541) at (0.2,0.943) {\small $(5,4,1)$};
  \node (631) at (0.05,0.898) {\small $(6,3,1)$};
  \node (622) at (-0.15,0.950) {\small $(6,2,2)$};
  \node (721) at (0,0.802) {\small $(7,2,1)$};
  \node (811) at (0,0.639) {\small $(8,1,1)$};
 \draw[ultra thick, red] (631) edge node[sloped]{} (622);
 \draw[ultra thick] (811) edge node[]{} (721);
 \draw[ultra thick] (721) edge node[right]{} (631);
 \draw[ultra thick] (622) edge node[right]{} (532);
 \draw[ultra thick] (532) edge node{} (442);
 \draw[ultra thick, blue] (442) edge node{} (433);
 \draw[ultra thick] (631) edge node{} (541);
 \draw[ultra thick] (541) edge node{} (532);
 \draw (721) -- (622);
 \draw (532) -- (433);
 \draw (631) -- (532);
 \draw (541) -- (442);
 \end{tikzpicture}

\caption{Diagram representing the partial order $\prec$ on $\mathcal{I}_{10}^3$. For example, $(8,1,1)\prec (7,2,1)$ which indicates $(7,2,1)$ is more balanced than $(8,1,1)$. The vertical position corresponds to the Shannon index $H$. Edges between two partitions represent the upper can be reached from the lower by a one-step downshift. Thick edges represent the covering relation  in \cref{def:order}(c). Especially, colored thick edges corresponds to $(**)$-covering with $n_u'=n_v'=2$ (red) and $n_u'=n_v'=3$ (blue); see \cref{subsec:3.3} and \cref{fig:balanceindexinterpret} for the detailed discussion on colored edges.}
\label{fig:order}
\end{figure}
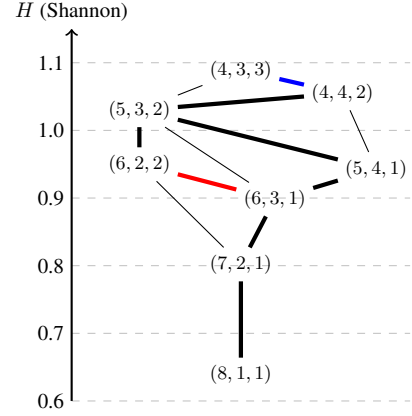

In this paper, we always assume finite exchangeability but not necessarily projectivity, as advocated in \citet{welling2006flexible}. Under the finite exchangeability, we can express the probability mass function of $\Pi_n$ using the \textit{exchangeable partition probability function} (EPPF) \citep{pitman1995exchangeable}. 
\begin{definition}[EPPF] A symmetric function $p^{(n)}$ is an EPPF of an exchangeable random partition $\Pi_n$ of $[n]$ if
\[
\mathbb{P}(\Pi_n = \{S_1,\ldots,S_k\}) = p^{(n)}(|S_1|, \ldots,|S_k|)
\]
\end{definition}

Since EPPF is symmetric of its arguments, the values of EPPF on integer partition $\bm{n} \in \mathcal{I}_n$ completely determine the probability mass function of random partition. 
Next, we introduce the \textit{Gibbs (or Gibbs-type) partition} \cite{pitman2006combinatorial},  a wide class of exchangeable random partition model.

\begin{definition}[Gibbs partition] Given two nonnegative sequences $\bm{V}=(V_{n,k})_{k\le n}$ and $\bm{W}=(W_s)_{s\ge 1}$ with $V_{1,1} = W_1 = 1$, we say $\Pi_n\sim \mathsf{Gibbs}_{[n]}(\bm{V},\bm{W})$ if its EPPF is
\begin{equation}
\label{eq:gibbspartition}
\textstyle
p^{(n)}(n_1,\ldots,n_k) = V_{n,k}\prod_{j=1}^k W_{n_j}
\end{equation}
\end{definition}
That is, the probability mass function has a product form that ensures finite exchangeability as well as mathematical tractability. While some authors define Gibbs partition with projectivity assumption (see \cref{prop:infinitegibbs}), our definition following \citet{pitman2006combinatorial} does not assume projectivity. The family of Gibbs partition is broad \citep[Fig.1]{lomeli2017marginal}, encompassing not only random partitions induced by DP, PYP, normalized generalized gamma process \citep{lijoi2007controlling}, and $\sigma$-stable Poisson-Kingman partitions \citep{pitman2003poisson}, but also those induced by parametric models such as Dirichlet-multinomial and MFM. 
To ensure Gibbs partition also satisfies projectivity (i.e., infinitely exchangeable), \citet{gnedin2006exchangeable,lijoi2007controlling} provided the following necessary and sufficient conditions.

\begin{proposition}
\label{prop1}
A $\mathsf{Gibbs}_{[n]}(\bm{V},\bm{W})$ partition for $n\in \mathbb{N}$, excluding partition of singletons and one-block partition, satisfies projectivity if and only if for some $\sigma \in [-\infty,1)$, the following two conditions are both satisfied:
\begin{enumerate}[label=(\roman*)]
    \item $W_{s} = \Gamma(s-\sigma)/\Gamma(1-\sigma)$,\quad $s=1,2,\ldots$ 
    \item $V_{n,k} = (n-\sigma k)V_{n+1,k} + V_{n+1,k+1}$, \quad $1\le k \le n$
\end{enumerate}
When $\sigma = -\infty$, replace $W_{s}\equiv1$ and $(n-\sigma k)$ by $k$ in (ii).
\label{prop:infinitegibbs}
\end{proposition}

The most important subclass of infinitely exchangeable Gibbs partitions is the \textit{Ewens-Pitman two-parameter family}  \citep{ewens1972sampling,pitman2006combinatorial} with parameters $(\sigma, \theta)$ where 
\begin{equation}
\label{eq:twoparameter}
    V_{n,k} = \frac{\prod_{i=1}^{k-1}(\theta+i\sigma)}{(\theta+1)\cdots(\theta+n-1)}, \quad W_s=\frac{\Gamma(s-\sigma)}{\Gamma(1-\sigma)}
\end{equation}
with exception of $V_{n,1}=1/((\theta+1)\cdots(\theta+n-1))$ for $n\ge 2$ and $V_{1,1}=1$. The range of $(\sigma,\theta)$ is either (a) $\sigma\in[0,1)$, $\theta>-\sigma$ which corresponds to the random partition induced by DP and PYP, (b) $\sigma<0$, $\theta=K|\sigma|$ for $K\in\mathbb{N}$ which corresponds to Dirichlet-multinomial allocation model, or (c) $\sigma=-\infty$, replacing $W_s\equiv 1$ and $V_{n,k}=K(K-1)\cdots(K-k+1)/K^n$ for some $K\in\mathbb{N}$ leads to the coupon-collector's partition. 
See \citet{pitman2006combinatorial, de2015gibbs} for the detailed treatments. 

\begin{figure*}[!h]
    \centering
    \includegraphics[width = 0.9\textwidth]{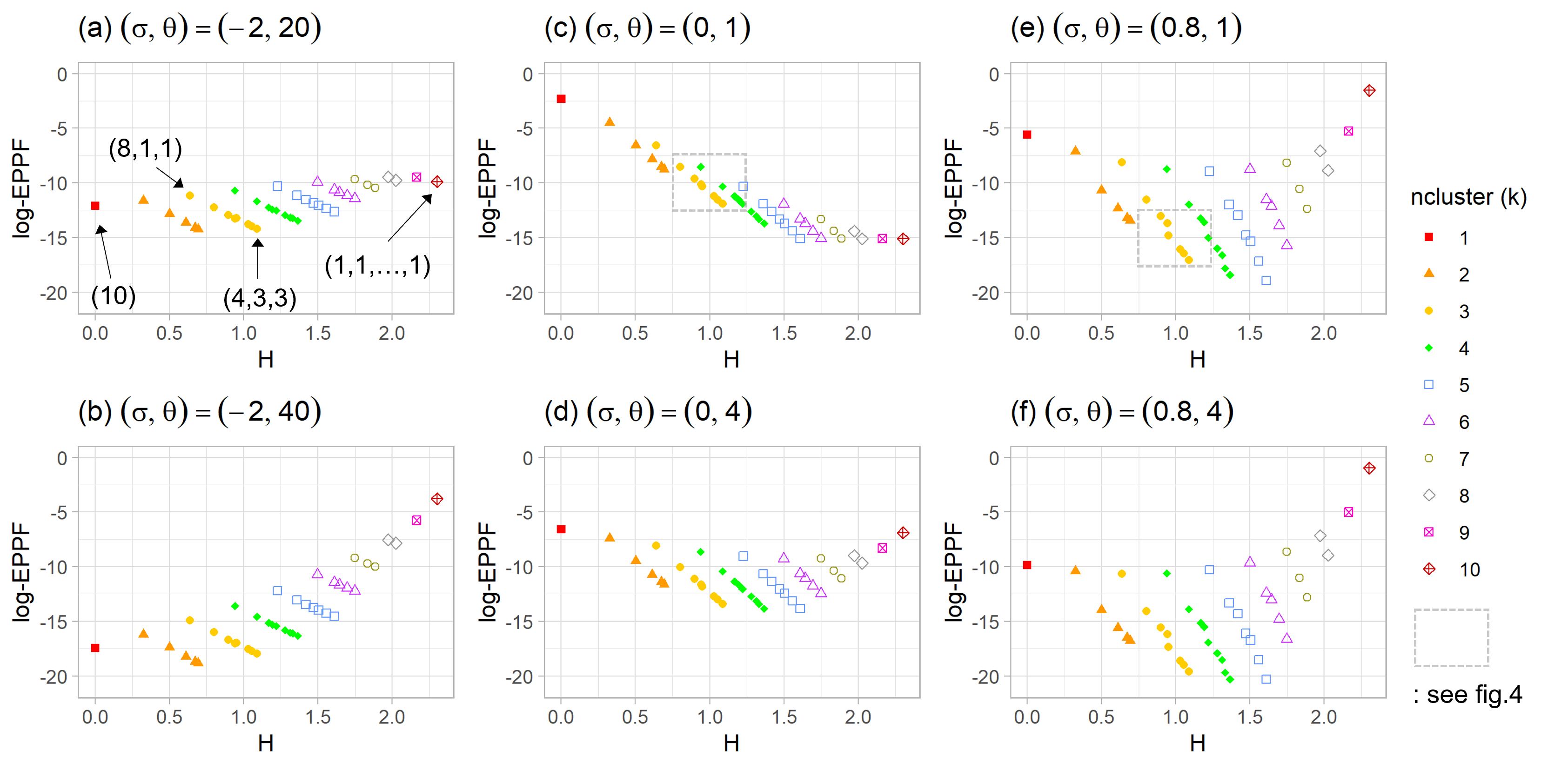}
    \caption{log-EPPF plots against Shannon index $H$ of Ewens-Pitman two-parameter models when $n=10$,  colored by the number of clusters. (a,b) Random partitions induced by Dirichlet-multinomial with symmetric Dirichlet parameter $|\sigma| = 2$ and number of components $K=10,20$ respectively. (c,d) Chinese restaurant process with concentration parameters $1$ and $4$ respectively. (e,f) Random partitions induced by Pitman-Yor process with common discount parameter $\sigma = 0.8$ and concentration parameters $1$ and $4$ respectively.}
    \label{fig:entropyspectrum}
\end{figure*}

\section{Balancedness of Gibbs Partition and Related Properties}
\label{sec:3}

\subsection{Why the Rich Get Richer?}
\label{subsec:3.1}
First, we define the balancedness of a finitely exchangeable random partition using the ordering $\prec$ in \cref{def:order}.

\begin{definition} \label{def:balancedness}Let $p^{(n)}$ be an EPPF of a finitely exchangeable random partition $\Pi_n$ on $[n]$. Then call $\Pi_n$ and $p^{(n)}$
\begin{itemize}[leftmargin=*]
    \item \textit{balance-averse} if $\bm{n}\prec \bm{n}'\implies p^{(n)}(\bm{n})\ge p^{(n)}(\bm{n}')$
    \item \textit{balance-seeking} if $\bm{n}\prec \bm{n}'\implies p^{(n)}(\bm{n})\le p^{(n)}(\bm{n}')$
\end{itemize}
for any fixed $k\le n$ and $\bm{n},\bm{n}'\in \mathcal{I}_n^k$. Also, call $\Pi_n$ \textit{balance-neutral} if $\Pi_n$ is both balance-averse and balance-seeking.
\end{definition}

The definition has a straightforward interpretation: $\Pi_n$ is balance-averse (seeking) when it assigns a lower (higher) probability to a more balanced partition in terms of $\prec$.

Now we present the necessary and sufficient conditions which characterize the balancedness of Gibbs partitions, where the proof is deferred to \cref{sec:appendixa}. 

\begin{theorem}
 Let $p^{(n)}$ be an EPPF of $\mathsf{Gibbs}_{[n]}(\bm{V},\bm{W})$. Then for any $n=1,2,\dots$, $p^{(n)}$ is
\begin{itemize}[leftmargin=*]
    \item balance-averse if and only if $\bm{W}$ is log-convex,
    \item balance-seeking if and only if $\bm{W}$ is log-concave,
\end{itemize}
where the sequence $\bm{W}=(W_s)_{s=1}^\infty$ is called log-convex~if $W_s^2 \le W_{s-1}W_{s+1}$ for $s\ge 2$, and is called log-concave~if $W_s^2 \ge W_{s-1}W_{s+1}$ for $s\ge 2$ and has no internal~zeros.
\label{thm:gibbs}
\end{theorem}

Thus, to check the balancedness of a Gibbs partition, it is not necessary to verify the conditions in \cref{def:balancedness}, and one only needs to analyze the log-convexity of $\bm{W}$. 

Gibbs partition admits a simple form of reallocation rule (prediction rule if projectivity holds) thanks to the product-form EPPF. 
The following corollary draws an explicit connection between the ``rich-get-richer'' metaphor and the balancedness of Gibbs partition based on \cref{def:balancedness}; see \cref{sec:appendixa} for the details.

\begin{corollary}
\label{coro:reallorule} Let $\Pi_{n+1}\sim \mathsf{Gibbs}_{[n+1]}(\bm{V},\bm{W})$. Given the cluster memberships of the first $n$ datapoints $\mathbf{z}_{1:n}$ with $k$ clusters, the reallocation rule for the next datapoint is
\begin{equation}
\label{eq:reallorule}
    \mathbb{P}(z_{n+1} =j |\mathbf{z}_{1:n}) \propto \begin{cases}
    f(n_j)& \text{if } j=1,\dots,k\\
    g(n,k)& \text{if }j=k+1
    \end{cases}
\end{equation}
Then $f$ is an increasing (decreasing) function over $\mathbb{N}$ if and only if $\Pi_{n+1}$ is balance-averse (seeking) for any $n\in\mathbb{N}$.
\end{corollary} 

In addition to the finite exchangeability, many existing random partition model also assumes projectivity (aforementioned examples of Gibbs partitions in \cref{subsec:2.3} are all infinitely exchangeable). 
By \cref{prop1}, the following corollary states that infinitely exchangeable Gibbs partitions are always balance-averse.

\begin{corollary}
\label{coro:infinitegibbs}
Let $p^{(n)}$ be an EPPF of infinitely exchangeable Gibbs partition. Then $p^{(n)}$ is always balance-averse; i.e., for two integer partitions $\bm{n},\bm{n}'\in \mathcal{I}_n^k$ with $k\le n$,
\[
\bm{n} \prec \bm{n}' \implies p^{(n)}(\bm{n}) \ge p^{(n)}(\bm{n}')
\]
with equality holds only if $\sigma = -\infty$. 
\end{corollary}
To see this, $W_{s}=\Gamma(s-\sigma)/\Gamma(1-\sigma)$, $s=1,2,\dots$ is strictly log-convex for any $\sigma\in (-\infty,1)$ and $W_{s} \equiv 1$ for~$\sigma = -\infty$.  
\Cref{fig:entropyspectrum} depicts the behavior of $\log p^{(n)}$ for the three different Ewens-Pitman two-parameter models in (\ref{eq:twoparameter}). 
The x-axis is the Shannon index $H$, which satisfies $\bm{n}\prec \bm{n}' \implies H(\bm{n})<H(\bm{n}')$. 
All scatterplots show a decreasing pattern for each $k$ as $H$ increases (i.e., partitions become more balanced), since two-parameter models are all infinitely exchangeable and thus balance-averse. 

\cref{coro:infinitegibbs} answers why the ``rich-get-richer'' property is shared across many existing random partition models. At the same time, it also implies that one should sacrifice projectivity to get a more flexible class of exchangeable random partition models with the balance-seeking property. 
The cost of sacrificing projectivity depends on the application. When the cluster membership of the future datapoint is a primary interest, lack of projectivity leads to an undesirable consequence that the implied joint distribution on the current dataset changes as we get more data. 
However, it has little impact if the primary goal is to make an inference on the given dataset such as entity resolution, and opens up a wide range of possible models with different balancedness.
\newpage 

\Cref{thm:gibbs} can be extended to the Gibbs partitions with additional hierarchical structure on $\bm{V}$ or $\bm{W}$. 
\begin{theorem} We say $\Pi_n$ is a mixture of Gibbs partition if it is finitely exchangeable and its EPPF $p^{(n)}$ has a form 
\begin{equation}
\label{eq:eppfmixturegibbs}
\textstyle
p^{(n)}(n_1,\ldots,n_k)=\int V_{n,k}(\vartheta)\prod_{j=1}^k W_{n_j}(\vartheta) \nu(d\vartheta)
\end{equation}
where $\vartheta$ is a mixing parameter which may be either discrete or continuous. Then for any $n=1,2,\ldots$, $p^{(n)}$ is
\begin{itemize}[leftmargin=*]
    \item balance-averse if $(W_s(\vartheta))$ is log-convex for each $\vartheta$,
    \item balance-seeking if $(W_s(\vartheta))$ is log-concave for each $\vartheta$.
\end{itemize}
\label{thm:mixgibbs}
\end{theorem}
It provides sufficient conditions for (\ref{eq:eppfmixturegibbs}) being balance-averse or balance-seeking; see \cref{sec:appendixa} for the proof. Examples include random partitions induced by DP with prior on the concentration parameter $\theta$, PYP with prior on the discount parameter $\sigma$ \citep{lijoi2007controlling}, and generalized MFM recently proposed by \citet{fruhwirth2021generalized}.

\subsection{Balance-neutral Random Partition Model}
\label{subsec:3.2}
There have been some efforts to develop random partitions that do not exhibit the ``rich-get-richer'' property \citep{jensen2008bayesian, wallach2010alternative} or control its rate \citep{lu2018reducing,poux2021powered} but at the cost of sacrificing exchangeability. Here we show that there does exist an exchangeable, projective random partition model that is balance-neutral, and \cref{thm:neutral} provides its only possible form of EPPF; see \cref{sec:appendixa} for the proof.

\begin{theorem}
\label{thm:neutral}
Let $p^{(n)}$ be an EPPF of balance-neutral infinitely exchangeable random partition. Then there exist some mixing distribution $q$ on the number of components $K\in\mathbb{N}$ so that
\begin{equation}
\label{eq:thmneutral}
\textstyle
    p^{(n)}(n_1,\dots,n_k) = \sum_{K=k}^\infty q(K)\frac{K(K-1)\cdots(K-k+1)}{K^n} 
\end{equation}
\end{theorem}
\cref{fig:balanceneutral} (left) gives an example when $q$ is $\mathsf{Poisson}(3)$ shifted by 1, showing the flat EPPF pattern for each $k$. It can be also viewed as the limiting case of the random partition induced by MFM when $|\sigma|\to\infty$. The prediction rule is 
\begin{equation}
\mathbb{P}(z_{n+1}=j|\mathbf{z}_{1:n}) \propto \begin{cases} 1 & \text{if  }j=1,\ldots,k\\ \frac{V_{n+1,k+1}(q)}{V_{n+1,k}(q)} &\text{if } j=k+1\end{cases}
\end{equation}

where $V_{n,k}(q)$ is the RHS of \eqref{eq:thmneutral} and these can be precomputed in advance; see \citet{miller2018mixture} for details. 
This balance-neutral random partition model can serve as a noninformative prior choice in terms of balancedness, as it assigns uniform probabilities for each $\Pi_n\in \mathcal{P}_{[n]}^k$ while $q$ controls the creation of new clusters.

\begin{figure}
    \centering
    \includegraphics[width = 0.45\textwidth]{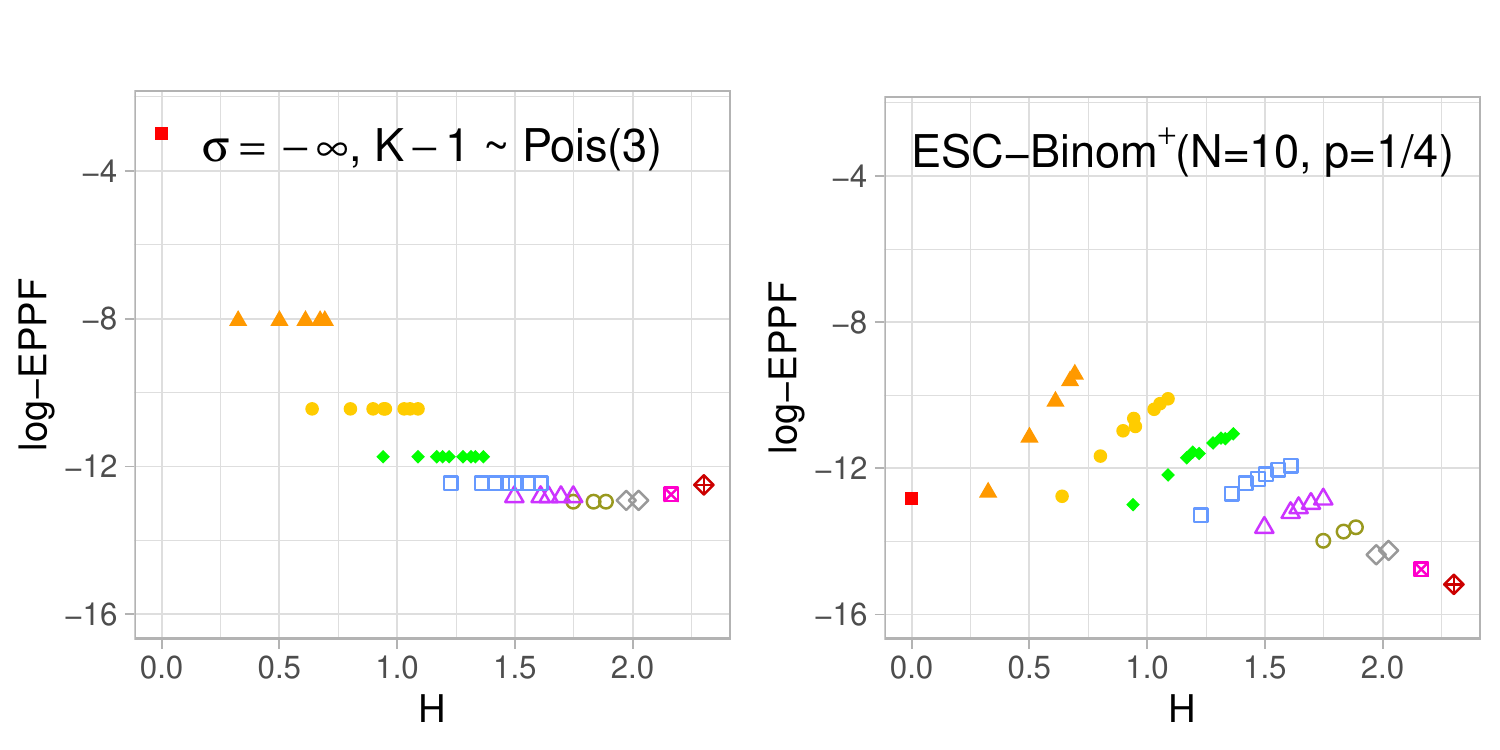}
    \caption{Two examples of log-EPPF plots, see \cref{fig:entropyspectrum} for the legend. (Left) Balance-neutral, when~$q=1+ \mathsf{Poisson}(3)$. (Right) Balance-seeking, when ESC model with $\bm\mu = \mathsf{Binom}^+(10,1/4)$. }
    \label{fig:balanceneutral}
\end{figure}

\subsection{Comparing the Strength of Balancedness with B-sequence}
\label{subsec:3.3}
\cref{thm:gibbs} indicates that the log-convexity/concavity of sequence $(W_s)_{s=1}^\infty$ determines the balancedness of Gibbs partition. Then a natural question arises: how to compare the balancedness between two random partitions? Which Gibbs partition is more balanced over the other?

Assuming $(W_s)$ has no internal zeros, we can use the concept of \textit{relative log-concavity ordering} $\lelc$ \citep{whitt1985uniform,yu2010relative} which is a preorder satisfying reflexivity and transitivity. 
$\bm{W}$ is called log-concave relative to $\bm{W}'$, written as $\bm{W}\lelc \bm{W}'$, if $\mathrm{supp}(\bm{W})\subseteq \mathrm{supp}(\bm{W}')$ and $(\log(W_s/W_{s}'))$ is concave in $\mathrm{supp}(\bm{W})$, the support of $\bm{W}$. 
While the comparison of $\bm{W}$ in ordering $\lelc$ itself can be used to compare the balancedness, there is also a need to define an intuitive measure that quantifies the balancedness so that we propose \textit{B-sequence} as follows. 

\begin{definition}[B-sequence]$(B_s(\bm{W}))_{s\ge 2}$ is a B-sequence of $\mathsf{Gibbs}_{[n]
}(\bm{V},\bm{W})$ for $n\in\mathbb{N}$, a sequence of extended real numbers which only depends on $\bm{W}$, defined as
\begin{equation}
    B_s(\bm{W}) = -s(\log W_{s+1}-2\log W_s + \log W_{s-1})
    \label{eq:balancedness}
\end{equation}
with the provision that $B_s(\bm{W})=+\infty$ if $W_{s+1}=0$. 
\end{definition}
This definition is closely related to the slope of log-EPPF against the Shannon index $H$. 
Recall the covering relation from \cref{def:order}(c) corresponding to the case $(**)$. If $\bm{n},\bm{n}'\in\mathcal{I}_n^k$ and $\bm{n}'$ is a $(**)-$cover of $\bm{n}$, then there exists some $s\ge 2$ such that $s=n_u'=n_u-1=n_v'=n_v+1$ with $u<v$. We prove that $nB_s(\bm{W})$ is an approximation of the slope along this covering relation as shown in \cref{fig:balanceindexinterpret}: 
\begin{equation}
    \frac{\log p^{(n)}(\bm{n}')-\log p^{(n)}(\bm{n})}{H(\bm{n}') - H(\bm{n})} \approx nB_s(\bm{W}),
    \label{eq:slope}
\end{equation}
where we defer the detailed derivation of approximation to \cref{sec:appendixb}. That is, B-sequence represents the strength of balancedness by measuring the derivative of log-EPPF against Shannon index $H$, the part which is invariant of $n$. 

Here $B_s(\bm{W})\equiv 0$ is equivalent to the Gibbs partition being balance-neutral which serves as the origin, and $B_s(\bm{W})\ge 0$ $(\le 0)$ for all $s$ is equivalent to the Gibbs partition being balance-seeking (balance-averse) respectively. 
The part of (\ref{eq:balancedness}) inside the parenthesis is the discrete second derivative of the sequence $(\log W_s)$ and naturally relates to log-convexity.
Indeed, we have $\bm{W}\lelc\bm{W}'$ if and only if $B_s(\bm{W})\ge B_s(\bm{W}')$ for all $s\ge 2$. See \cref{sec:appendixb} for the proof.

\begin{figure}
    \centering
    \includegraphics[width = 0.47\textwidth]{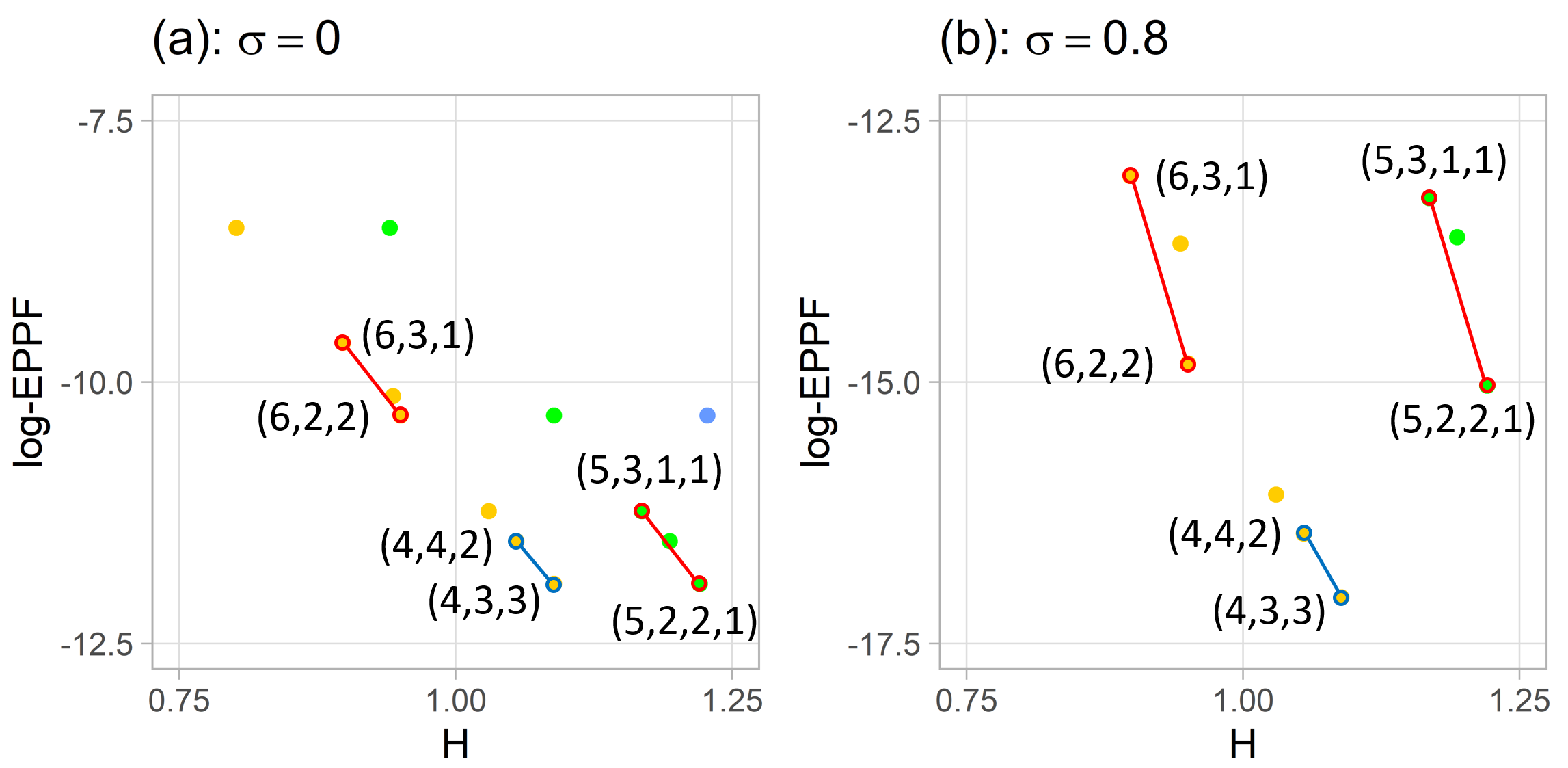}
    \caption{Graphical interpretation of B-sequence via (\ref{eq:slope}), where $n=10$ and $\bm{W}$ in \eqref{eq:twoparameter} is parameterized by $\sigma$. Subplots (a), (b) are zoomed portion of \cref{fig:entropyspectrum}(c), (e) respectively. Red segments correspond to $(**)$-covering relation with $s=2$, and blue with $s=3$ (see also \cref{fig:order}). Approximately, the slope of red segments~are $nB_2(\sigma=0)$ (left) and $nB_2(\sigma=0.8)$ (right), and blue segments are $nB_3(\sigma=0)$ (left) and $nB_3(\sigma=0.8)$ (right). 
    As $\sigma$ increases, slopes become steeper and the EPPF becomes more unbalanced.}
    \label{fig:balanceindexinterpret}
\end{figure}

\textbf{Examples}. 
Consider Ewens-Pitman two-parameter family \eqref{eq:twoparameter} where $\bm{W}$ is parameterized by $\sigma$. Then, for $\sigma<\sigma'<1$, the corresponding B-sequences $(B_s(\sigma))$ and $(B_s(\sigma'))$ satisfy $B_s(\sigma')<B_s(\sigma)<0$ for all $s\ge 2$. 
It implies that Dirichlet-multinomial or MFM ($\sigma<0$) based random partition model is more balanced than CRP ($\sigma=0$) induced by DP, and CRP is more balanced than the one induced by PYP ($\sigma\in(0,1)$). %
The slope interpretation in (\ref{eq:slope}) well matches with the pattern shown in \cref{fig:entropyspectrum}; the negative slope becomes steeper as $\sigma$ increases. 
This result is also consistent with the comparisons in \citet{green2001modelling,miller2018mixture} on the mixture modeling tasks, where CRP typically yields small extraneous clusters compared to Dirichlet-multinomial or MFM.

The effect of balancedness can be also found in topic model applications. Although the difference between LDA \citep{blei2003latent} and hierarchical DP \citep[HDP;][]{teh2006hierarchical} is often emphasized in terms of the limited/unlimited number of topics, it is almost always neglected that the prior on the per-document topic distribution of HDP is more unbalanced than those of LDA. As an example of the effect of balancedness, \citet{shi2019new} studied how stopwords (e.g. ``a'', ``the'', \ldots) affect the performance of topic modeling algorithms. 
As the proportion of stopwords increases, the performance of HDP deteriorates significantly compared to LDA since stopwords make the per-document topic proportion more balanced. It also suggests the importance of the hyperparameter choice which affects the balancedness of the model \citep{asuncion2009smoothing}.

For the community detection problems, \citet{legramanti2022extended} utilized infinitely exchangeable Gibbs partitions as a prior on the partition of nodes. In their simulation studies with true clusters being fairly balanced, a random partition prior with $\sigma=-1$ \citep{gnedin2010species} performed the best compared to other more unbalanced priors with $\sigma>-1$. Also see \citet{schmidt2013nonparametric} for a real network data application of the CRP prior \citep{kemp2006learning}, which yields many tiny community estimates. 
To sum up, the balancedness of random partition models can help understand these findings and choose random partition models with desired balancedness to tackle these problems.

\textbf{Remark}. The comparison of balancedness between some special cases of random partitions has been discussed before \citep{miller2018mixture,fruhwirth2021generalized} but from a different perspective. They considered the aggregated probabilities $\sum_{\Pi_n\in \mathcal{A}}\mathbb{P}(\Pi_n = \{S_1,\ldots,S_k\})$ with $\mathcal{A}= \{\Pi_n\in\mathcal{P}_{[n]}^k:\Pi_n \text{ has cluster sizes }(n_1,\ldots,n_k)\}$ (or equivalently, the distribution of \textit{labeled cluster sizes}), instead of the probability mass of each partition $\mathbb{P}(\Pi_n = \{S_1,\ldots,S_k\}) = p^{(n)}(|S_1|,\ldots,|S_k|)$, the EPPF. 
Comparison of aggregated probabilities has a somewhat unnatural 
interpretation when it comes to explaining the prior effect of random partition models on the posterior estimate. The probability mass of individual partition (EPPF) is what contributes to the posterior, not the aggregated probability. 
While the distribution of functionals over the labeled cluster sizes, such as mean and variance of the Shannon index \citep{Greve2022Spying}, could be indicative of comparison regarding balancedness, comparison of EPPF provides a more direct interpretation of the prior effect on the posterior estimate which is often a point of interest.

\section{Balance-seeking Random Partition Models}
\label{sec:4}

\subsection{Entity Resolution and ESC Model}
\label{sec:4.1}

Entity resolution (ER) is the process of matching records that describe the same individual when there are no unique identifiers available (such as due to privacy reasons). 
Assuming each record corresponds to one individual, ER can be understood as a clustering task where each cluster corresponds to each individual; see \citet{binette2022almost} for the most recent review of ER literature.

Databases have many noisy records, where reasons include address changes, name changes, measurement/transcription errors among many others. 
To deal with such distortions with uncertainty quantification, \citet{steorts2016bayesian} developed a full Bayesian hierarchical model, and we briefly describe the model here. 
Assume $n$ records $(\bm{x}_i)_{i=1}^n$ contain $L$ categorical fields with entries $(x_{i\ell})$. Let $D_\ell$ be the number of categories in the $\ell$th field and let the corresponding probabilities $\bm\theta_{\ell}=(\theta_{\ell1}, \ldots, \theta_{\ell D_{\ell}})$.  
Also let the cluster indicators be $(z_i)_{i=1}^n$, where the total number of individuals $K^+$ is unknown. Assuming the latent entities $(y_{k\ell})_{k=1}^{K^+}$ are drawn from $\mathsf{Categorical}(\bm\theta_\ell)$, the generative process of entries is $x_{i\ell} = y_{z_i\ell}$ with probability $1-\beta_{\ell}$ (non-distorted), and $x_{i\ell}\sim \mathsf{Categorical}(\bm\theta_\ell)$ with probability $\beta_\ell$ (distorted).
See \cref{table:database} for an example. With this Bayesian hierarchical model, by assigning a prior on the partition represented by cluster indicators $(z_i)$, one can obtain a posterior estimate of $z_i$ with uncertainty quantification of record matches.

\begin{table}
\caption{An example database with categorical entities $(x_{i\ell})$. Here $z_i$ are assumed to be unknown. Reds are `distorted' entities, modeled with distortion probabilities $(\beta_\ell)$ and latent entities $(y_{k\ell})$. }
\label{table:database}
\vskip 0.1in
\begin{center}
\begin{small}
\begin{sc}
\begin{tabular}{cccccc}
\toprule
$i$ & Sex & Surname & State & & $z_i$  \\
\midrule
1& F &\textcolor{red}{Smith} & CA & $\cdots$ & 1 \\
2& F &Johnson & CA & $\cdots$ & 1 \\
3& F &Johnson & CA & $\cdots$ & 1 \\
4& M &Williams & TX & $\cdots$ & 2 \\
5& M &Williams & TX & $\cdots$ & 2 \\
6& M &Williams & \textcolor{red}{FL} & $\cdots$ & 2\vspace{-1mm} \\
$\vdots$ &$\vdots$ & $\vdots$ &$\vdots$& & $\vdots$ \\
\bottomrule
\end{tabular}
\end{sc}
\end{small}
\end{center}
\vskip -0.1in
\end{table}

However, many traditional random partition models assume infinite exchangeability and thus are not suitable for ER tasks, since it leads to the behavior that the number of datapoints in each cluster grows linearly with $n$, by Kingman's paintbox representation theorem \citep{kingman1978representation}. 
To overcome this problem, \citet{miller2015microclustering,zanella2016flexible} formalized the \textit{microclustering property}: $M_n/n \stackrel{p}{\to} 0$ where $M_n$ is the maximum size of clusters so that cluster sizes grow sublinearly with $n$.
Recently \citet{betancourt2020random} proposed the exchangeable sequences of clusters (ESC) model, which provides a very general framework of random partition models possessing the microclustering property under mild assumptions. It sacrifices projectivity, and the key idea is to consider a discrete probability measure $\bm{\mu}=(\mu_s)_{s=1}^\infty$ with $\mu_1>0$ that represents the distribution of the cluster sizes $n_j=|S_j|$, $j=1,\ldots,k$. Conditional on $\bm\mu$, ESC is a Gibbs partition model where its EPPF is:
\begin{equation}
    p^{(n)}(n_1,\ldots,n_k|\bm{\mu}) = \frac{1}{\mathbb{P}(E_n|\bm{\mu})}\frac{k!}{n!}\prod_{j=1}^k n_j!\mu_{n_j} 
    \label{eq:esceppf}
\end{equation}
where $E_n := \left\{\text{there exists }k\in\mathbb{N}\text{ such that }\sum_{j=1}^k n_j = n\right\}$. 
When $\sum_{s=1}^\infty s\mu_s<\infty$, a random partition with EPPF of form (\ref{eq:esceppf}) has the microclustering property. 

\subsection{Balance-seeking ESC Models}
\label{sec:4.2}
The Balancedness of random partition models also plays an important role in ER tasks, since it is not plausible to assume that a small number of clusters dominate the whole database.
For example, assume there are $n=1000$ number of records of $K^+=100$ individuals. Then a balance-averse random partition model assigns a higher probability to a partition with cluster sizes $(901,1,\ldots,1)$ (length 100) compared to a partition with cluster sizes $(10,10,\ldots,10)$, suggesting the balance-seeking model would be better suited. 
The microclustering property has a similar rationale by limiting the growth rate of the maximum cluster size.  
While the microclustering property is based on the asymptotic behavior of cluster sizes, balancedness focuses on the non-asymptotic point of view and they complement each other.

We expand the current knowledge of ESC models by providing subclasses of ESC models
with different balancedness properties. By \cref{thm:gibbs} and (\ref{eq:balancedness}), $\bm\mu$ being the zero-truncated Poisson distribution in \eqref{eq:esceppf} is the choice such that the B-sequence becomes identically zero, leading to the following theorem; see \cref{sec:appendixa} for the proof. 

\begin{theorem}
 Let $p^{(n)}$ be an EPPF of ESC model with fixed measure $\bm\mu$. Then for any $n=1,2,3\cdots$,
\begin{itemize}[leftmargin=*]
    \item $p^{(n)}$ is balance-averse if and only if $\bm\mu \gelc \mathsf{Poisson}^+$,
    \item $p^{(n)}$ is balance-seeking if and only if $\bm\mu\lelc \mathsf{Poisson}^+$,
    \item $p^{(n)}$ is balance-neutral if and only if $\bm\mu=\mathsf{Poisson}^+$.
\end{itemize}
where $\mathsf{Poisson}^+$ indicates the p.m.f. of zero-truncated Poisson distribution with arbitrary parameter $\lambda$.
\label{thm:esc}
\end{theorem}

In \cref{table:muexamples}, we provide various examples of $\bm\mu$ that lead to balance-averse, neutral, and balance-seeking ESC models. 
See \cref{fig:balanceneutral} (right) for an example of log-EPPF plot of balance-seeking model when $\bm\mu= \mathsf{Binomial}^+(10,1/4)$, showing the increasing pattern for each $k$. 
\Cref{thm:esc} implies that a balance-seeking ESC model always possesses the microclustering property since all $\bm\mu$ which are relative log-concave to Poisson have finite moments of all orders \citep{johnson2013log}. 
Especially, $\bm\mu = \mathsf{Binomial}^+$ with fixed number of trial $N$ has the \textit{bounded microclustering property} \citep{betancourt2022prior}, the size of the largest cluster is upper bounded.

With the Bayesian hierarchical model for ER task outlined in \cref{sec:4.1}, one can choose a family of $\bm\mu$ to construct an ESC partition prior and carry out the posterior inference with Markov chain Monte Carlo (MCMC) methods. This includes a posterior sampling of cluster indices $z_i$ as well as $\theta_\mu$ (parameters of $\bm\mu$) from their conditional distributions.
For example, the full conditional distribution of $z_i$ is
\begin{equation*}
\mathbb{P}(z_{i}=j | -) \propto p(\bm{x}|\mathbf{z}_{-i}, z_{i}=j,-) \mathbb{P}(z_i=j|\mathbf{z}_{-i},\theta_\mu)
\end{equation*}
where the latter part $p^{-i}_{j}:=\mathbb{P}(z_i=j|\mathbf{z}_{-i},\theta_\mu)$ is
\begin{equation*}
p^{-i}_j\propto\begin{cases}
\frac{(n_j^{-i}+1)\mu_{n_j^{-i}+1}}{\mu_{n_j^{-i}}} \stackrel{let}{=} f(n_j^{-i};\theta_\mu) &\text {if } j=1, \ldots, k^{-i}, \\
\left(k^{-i}+1\right)\mu_1 & \text {if } j=k^{-i}+1
\end{cases}
\end{equation*}
where $k^{-i}$ and $n_j^{-i}$ are the number and sizes of clusters in $\mathbf{z}_{-i}$.
In \cref{sec:appendixc}, we provide details of EPPFs, reallocation rules and posterior inference algorithms of five different examples of $\bm\mu$ in \cref{table:muexamples}. 

Notably when $\bm\mu = \mathsf{Binomial}^+(N,p)$, the corresponding ESC model is balance-seeking and the reallocation probability to an existing cluster is proportional to $f(n_j^{-i};\theta_\mu)=N-n_j^{-i}$ for $n_j^{-i}\le N$, a linear function of $n_j^{-i}$ with a negative slope showing the ``rich-get-poorer'' characteristic.

When $n$ is very large, the standard Gibbs sampling algorithm which reallocates $z_i$ one by one may suffer from slow mixing of the Markov chain. We utilize the \textit{chaperones algorithm} \citep{miller2015microclustering,zanella2016flexible} which focuses on reallocations that have higher probabilities, but adopting other strategies (e.g. split-merge sampler) is also possible.
We remark that the change of balancedness of Gibbs partition will not affect the computational cost as long as $\bm{V},\bm{W}$ are readily available.
More details on posterior inference algorithms can be found in \cref{sec:appendixc}.

\begin{table}
\caption{Examples of $\bm\mu$ in ESC model and their supports, grouped by balancedness property. Superscript $+$ indicates zero-truncated distribution. $\mathsf{CMP}$ stands for a Conway-Maxwell-Poisson distribution \citep{shmueli2005useful} with parameters $\lambda$ and $\nu$.}
\label{table:muexamples}
\vskip 0.1in
\begin{center}
\begin{small}
\begin{tabular}{cccc}
\toprule
 & $\bm\mu$ & Support \\
\midrule
\multirowcell{3}{Balance-\\averse\\($B_s<0$)}& $\mathsf{NegativeBinomial}^+$ & $\mathbb{N}$\\
 & $\mathsf{Geometric}$ & $\mathbb{N}$\\
 & $\mathsf{Logarithmic}$ & $\mathbb{N}$ \\
\midrule
Neutral & $\mathsf{Poisson}^+$ & $\mathbb{N}$\\
\midrule
\multirowcell{3}{Balance-\\seeking\\$(B_s>0)$} & $\mathsf{Binomial}^+$ & finite\\
& $\mathsf{HyperGeometric}^+$  & finite, must contain 1\\
& $\mathsf{CMP}^+$ with $\nu>1$ & $\mathbb{N}$ \\ 
\bottomrule
\end{tabular}
\end{small}
\end{center}
\vskip -0.1in
\end{table}

\subsection{Real Data Application of Balance-seeking Models}
\label{sec:4.3}

In this section, we demonstrate the effectiveness of balance-seeking random partition for the ER task using the Survey of Income and Program Participation data \citep{sipp1000}. 
We use the same dataset (SIPP1000) that \citet{betancourt2020random} used to benchmark the performance; the database with $n=4116$ (number of records) and $K^+=1000$ (number of entities) was collected from the five waves of the longitudinal survey performed between 2005-2006. 
The main task is to recover the identifiers only using the $L=5$ categorical fields (sex, birth year, birth month, race, and state of residence), assuming the unique identifiers are unknown and compare the results with the true partition. 

With the SIPP1000 dataset and the hierarchical model described in \cref{sec:4.1}, we compare the performance of four different ESC models where  $\bm\mu$ is: (i) $\mathsf{Binomial}^+$ with a fixed number of trials $N=5$ (balance-seeking), (ii) $\mathsf{Poisson}^+$ (balance-neutral), (iii) $\mathsf{NegativeBinomial}^+$ (balance-averse), and (iv) $\mathsf{Dirichlet}$ \citep[neither balance-seeking nor averse;][]{betancourt2020random}. 
Hyperpriors and MCMC specification details are described in \cref{sec:appendixd}.

\begin{figure}[th]
    \centering
    \includegraphics[width = 0.48\textwidth]{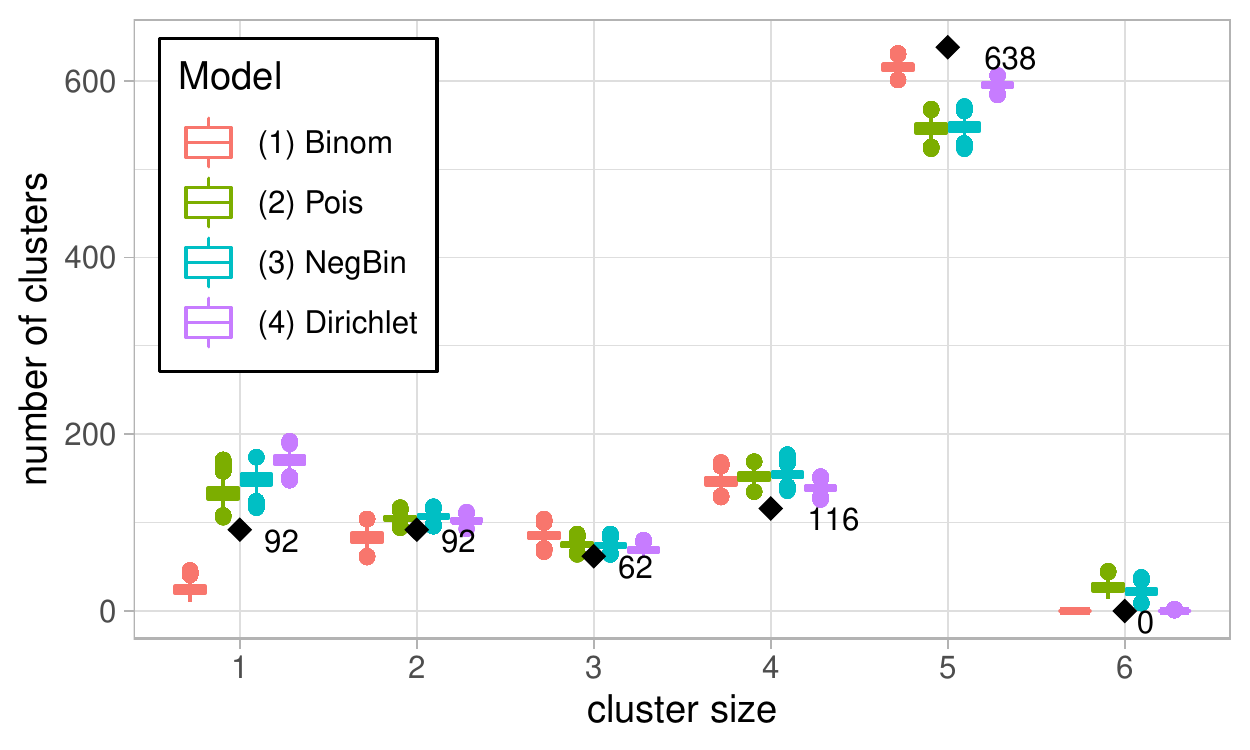}
    \caption{Posterior distributions (boxplots) of the number of clusters of size $i$ for each models, $m_i$, for the SIPP1000 data. True $m_i$ are shown as black diamonds and annotated by their values.}
    \label{fig:sippboxplot}
\end{figure}
\begin{table}[th]
\caption{SIPP1000 data: posterior mean and standard deviation of the number of entities ($K^+_{true}\!=\!1000$); FNRs and FDRs in \%.\protect\footnotemark}
\label{table:SIPP}
\vskip 0.1in
\begin{center}
\begin{small}
\begin{tabular}{cccccc}
\toprule
$\bm{\mu}$ (balance-) & $\mathbb{E}(K^+|\bm{x})$& SD & FNR & FDR \\
\midrule
$\mathsf{Binom}^+$ (seeking) & 951.5& 6.2 & \textbf{4.3} &  4.6\\
$\mathsf{Poisson}^+$ (neutral) & 1043.8& 8.8 & 4.9 & 3.3\\
$\mathsf{NegBin^+}$ (averse) & 1058.6& 8.8 & 4.9 & 3.0\\
$\mathsf{Dirichlet}$ (n/a) & 1076.4 & 5.3 & 4.8 & \textbf{1.5} \\ 
\bottomrule
\end{tabular}
\end{small}
\end{center}
\vskip -0.1in
\end{table}
\footnotetext{FNR and FDR are calculated based on the true partition and the point estimate obtained with \citet{dahl2006model}'s method. Let CP, MP, WP be the number of correct, missed, and wrong pairs respectively, then FNR $=$ MP$/$(MP$+$CP) and FDR $=$ WP$/$(WP$+$CP).}

The result in \cref{fig:sippboxplot} shows that the balance-seeking random partition prior penalizes the number of small-size clusters such as singletons, while the balance-averse prior does the opposite. 
This is evident since the balance-seeking model assigns a less prior probability to the partition with many singleton clusters where singletons make a partition unbalanced.
Consequently, in \cref{table:SIPP}, the balance-seeking model gives a more conservative estimate of the number of clusters with a lower false negative rate (FNR), but at the cost of yielding a higher false discovery rate (FDR).

The results indicate that balance-seeking random partition models are more effective in false negative control. Indeed, in many practical applications of ER, FNR control carries more weight than FDR control \citep{randall2013effect}. 
For example, if the matched record is used to inform patients of their conditions, it is much more important to reduce FNR (fail to inform) than FDR (inform to the wrong patient). 
In \cref{sec:appendixd}, we present additional simulation results with various different cluster size distributions and different distortion probabilities ($\beta_\ell$).
The results show that the balance-seeking model generally achieves a lower FNR rate compared to balance-neutral or balance-averse models.

\section{Concluding Remarks}

In this paper, we rigorously define and study the balancedness property of exchangeable random partition models, especially focusing on Gibbs partition models. 
By characterizing the ``rich-get-richer'' property as a balance-averseness, \cref{coro:infinitegibbs} provides an insight that two common assumptions (product-form exchangeability and projectivity) of random partition models lead to the ``rich-get-richer'' property. 
Although we conjecture that \textit{every} infinitely exchangeable random partitions are balance-averse, we remark that our result is very general because almost all existing random partition models are either Gibbs or a mixture of Gibbs due to the simple reallocation rule \eqref{eq:reallorule}.
There is emerging literature beyond the Gibbs partition framework \cite{favaro2011class,camerlenghi2018bayesian} and analyzing the balancedness of such models is an interesting future direction to pursue.

Another interesting future direction is to study the balancedness of non-exchangeable random partition models. Examples include random partitions with covariate information~\citep{park2010bayesian,muller2011product},
spatial or pairwise information~\citep{blei2011distance,page2016spatial,dahl2017random,xie2020bayesian}, baseline partition serves as a ``center''~\citep{smith2019demand,paganin2021centered}, temporal contiguity~\citep{barry1992product,barry1993bayesian, monteiro2011product} and spatial contiguity \citep{teixeira2019bayesian,luo2021bayesian,lee2021t,luo2021bast}. When random partition is no longer exchangeable, probability mass cannot be represented as an EPPF, and the extension of \cref{def:balancedness} is highly nontrivial.

To conclude, we characterize the complete family of balance-neutral random partitions, which can be used as a noninformative prior choice in terms of balancedness. We also propose \textit{B-sequence} to represent and compare the balancedness in a more principled and intuitive way.
We offer various flexible balance-seeking random partition models, which provide better modeling solutions for applications where traditional balance-averse models are not suitable and prediction is not the main interest (due to the lack of projectivity).
These models regularize the emergence of small-size clusters, which is useful in tasks such as entity resolution.
We hope this article can serve as a useful guide for researchers to better understand the different behaviors of random partition models and practitioners to decide on a suitable model according to specific applications.



\section*{Acknowledgements}
The research of Changwoo Lee and Dr. Huiyan Sang was supported by NSF grant no. NSF DMS-1854655. We thank anonymous reviewers, Hyunwoong Chang, and Dr. Yang Ni for helpful comments. We also thank Professor Rebecca C. Steorts for sharing the preprocessed SIPP1000 dataset.


\newpage 
\bibliography{balancedness}
\bibliographystyle{icml2022}

\newpage
\appendix
\onecolumn

In \cref{sec:appendixa}, we provide the proofs of all theorems in this paper: \cref{thm:gibbs}, \cref{coro:reallorule}, \cref{thm:mixgibbs}, \cref{thm:neutral}, and \cref{thm:esc}. In \cref{sec:appendixb}, we provide a detailed analysis of the B-sequence on its graphical interpretation and properties regarding balancedness.
In \cref{sec:appendixc}, we provide more details on the subclasses of the ESC model discussed in \cref{sec:4.2} and their corresponding inference algorithms using MCMC methods. Finally, in \cref{sec:appendixd}, we provide details on the SIPP1000 data analysis results in \cref{sec:4.3} and also present the simulation study results. In the supplementary material, we provide an interactive demo regarding the balancedness of the Ewens-Pitman two-parameter family and R code to run the balance-seeking ESC model described in \cref{sec:4.2} and \cref{sec:appendixc}.

\section{Proofs of Theorems}
\label{sec:appendixa}
\subsection{Proof of Theorem 3.2}

(Log-convexity of $\bm{W}$ implies balance-averseness) Let $\bm{n},\bm{n}'\in\mathcal{I}_n^k$ be two distinct integer partitions of $n$ into $k$ parts. To show that $\bm{n} \prec \bm{n}' \implies p^{(n)}(\bm{n})\ge p^{(n)}(\bm{n}')$, by the transitivity of $\prec$ and $\le$, it is sufficient to show for the case when $\bm{n}'$ can be reached by a single one-step downshift from $\bm{n}$. 
Let $\bm{n} = (n_1,\cdots,n_k)$ and its one-step downshift $\bm{n}' = (n_1',\cdots,n_k') = (n_1,\cdots n_u-1,\cdots, n_v+1,\cdots n_k)$ for some $1\le u<v \le k$ such that $n_u-1\ge n_v+1$. Then the log-convexity of $\bm{W}$, i.e. convexity of the sequence $(\log W_s)_{s=1}^{\infty}$ implies 
\begin{align*}
    & \log W_{n_v+1} - \log W_{n_v} \le  \log W_{n_u} - \log W_{n_u-1} & \because n_u-1 \ge n_v+1\\
    \implies&  \log W_{n_u} + \log W_{n_v} \ge \log W_{n_u-1} + \log W_{n_v+1}\\
    \implies& W_{n_u}W_{n_v} \ge W_{n_u-1}W_{n_v+1}=W_{n_u'}W_{n_v'}\\
    \implies& p^{(n)}(\bm{n}) = V_{n,k}\prod_{j=1}^k W_{n_j} \ge V_{n,k}\prod_{j=1}^k W_{n_j'} =  p^{(n)}(\bm{n}')
\end{align*}
Note that $W_s>0$ for all $s$, since $W_1=1$ and $W_s^2\le W_{s-1}W_{s+1}$ for all $s$. To see this, if $W_s=0$ for some $s=s^*$, then it implies $W_s=0$ for all $s\le s^*$, which contradicts to $W_1=1$.

(Balance-averseness implies log-convexity of $\bm{W}$) Let $\bm{n} = (n_1,\cdots,n_k)$ and $\bm{n}' = (n_1,\cdots,n_u-1,\cdots,n_v+1, \cdots,n_k)$ with $n_u-1 = n_v+1 \stackrel{let}{=} n^*$. Then $p^{(n)}(\bm{n})\ge p^{(n)}(\bm{n}')$ by balance-averseness, and this implies $W_{n^*}^2\le W_{n^*+1}W_{n^*-1}$. Since $n^*\in\{2,3,\ldots\}$ can be chosen arbitrarily as balance-averseness holds for any $k$ and $n$ such that $k\le n$, thus $\bm{W}$ is log-convex. 

The equivalency between balance-seeking and log-concavity follows with similar arguments.

\subsection{Proof of Corollary 3.3}
Let $\Pi_{n+1}\sim \mathsf{Gibbs}_{[n+1]}(\bm{V},\bm{W})$ so that the EPPF is $p^{(n+1)}(n_1,\dots,n_k) = V_{n+1,k}\prod_{l=1}^k W_{n_l}$. Let the cluster membership of first $n$ datapoints with $k$ clusters be $\mathbf{z}_{1:n}\in\{1,\dots,k\}^n$ and $n_l = \sum_{i=1}^n 1(z_i=l)$ for $l=1,\dots,k$.
We first derive the expression \eqref{eq:reallorule}, using $W_1=1$ by definition,

\begin{align*}
\mathbb{P}(z_{n+1}=j|\mathbf{z}_{1:n}) &\propto \mathbb{P}(z_{n+1}=j, \mathbf{z}_{1:n})\\
&= \begin{cases} 
V_{n+1,k} W_{n_j+1}\prod_{l\neq j} W_{n_l} &\text{if } j=1,\dots,k\\ 
V_{n+1,k+1} \prod_{l=1}^k W_{n_l} &\text{if }j=k+1
\end{cases}\\
&\propto \begin{cases} 
W_{n_j+1}/W_{n_j} \stackrel{let}{=} f(n_j) &\text{if } j=1,\dots,k\\ 
V_{n+1,k+1}/V_{n+1,k} \stackrel{let}{=} g(n,k)&\text{if }j=k+1
\end{cases}
\end{align*}
For any $n\in\mathbb{N}$, reallocation probability to the existing cluster $f(n_j) = W_{n_j+1}/W_{n_j}$ is an increasing (decreasing) function of $n_j=1,2,\dots$ if and only if $\bm{W}=(W_s)_{s=1}^\infty$ is log-convex (log-concave). Combined with \cref{thm:gibbs}, we have the \cref{coro:reallorule}. 

\subsection{Proof of Theorem 3.5}
Let $p^{(n)}$ be an EPPF of mixture of Gibbs partition of the form in (\ref{eq:eppfmixturegibbs}), i.e., 
$$
p^{(n)}(n_1,\ldots,n_k)=\int_\Theta V_{n,k}(\vartheta)\prod_{j=1}^k W_{n_j}(\vartheta) \nu(d\vartheta)
$$
Assume that $(W_s(\vartheta))_{s=1}^\infty$ is log-convex for each $\vartheta\in \Theta$, where $\Theta$ is the domain of $\vartheta$ that is either discrete or continuous. Fix $\vartheta$ at an arbitrary value, and let $\bm{n},\bm{n}'\in\mathcal{I}_n^k$. 
To show $\bm{n} \prec \bm{n}' \implies p^{(n)}(\bm{n})\ge p^{(n)}(\bm{n}')$, by the transitivity of $\prec$ and $\le$, it is again sufficient to show for the case when $\bm{n}'$ can be reached by a single one-step downshift from $\bm{n}$. 
Let $\bm{n} = (n_1,\cdots,n_k)$ and $\bm{n}' = (n_1',\cdots,n_k') = (n_1,\cdots n_u-1,\cdots, n_v+1,\cdots n_k)$ for some $1\le u<v \le k$ such that $n_u-1\ge n_v+1$. Then the convexity of sequence $(\log W_s(\vartheta))_{s=1}^\infty$ implies 
\begin{align*}
    & \log W_{n_v+1}(\vartheta) - \log W_{n_v}(\vartheta) \le  \log W_{n_u}(\vartheta) - \log W_{n_u-1}(\vartheta) & \because n_u-1 \ge n_v+1\\
    \implies&  \log W_{n_u}(\vartheta) + \log W_{n_v}(\vartheta) \ge \log W_{n_u-1}(\vartheta) + \log W_{n_v+1}(\vartheta)\\
    \implies& W_{n_u}(\vartheta)W_{n_v}(\vartheta) \ge W_{n_u-1}(\vartheta)W_{n_v+1}(\vartheta)=W_{n_u'}(\vartheta)W_{n_v'}(\vartheta)\\
    \implies& V_{n,k}(\vartheta)\prod_{j=1}^k W_{n_j}(\vartheta) \ge V_{n,k}(\vartheta)\prod_{j=1}^k W_{n_j'}(\vartheta) \ge 0,
\end{align*}
where the last line holds for arbitrary $\vartheta\in\Theta$. Since integration preserves inequality of functions, we have $p^{(n)}(\bm{n})\ge p^{(n)}(\bm{n}')$. The statement that log-concave $(W_s(\vartheta))_{s=1}^\infty$ implies a balance-seeking $p^{(n)}$ follows with similar arguments.

\subsection{Proof of Theorem 3.6}
Let $p$ be an EPPF of balance-neutral, infinitely exchangeable random partition. Note that any $\bm{n}\in \mathcal{I}_n^k$ is comparable with the most unbalanced one $\bm{n}^\star=(n-k+1,1,\cdots,1)$ and the most balanced one $\bm{n}^{\star\star}=(\underbrace{\lceil n/k\rceil,\cdots, \lceil n/k\rceil}_{(n \operatorname{mod} k) \text{ times}}, \lfloor n/k\rfloor, \cdots, \lfloor n/k\rfloor )$, and  $\bm{n}^\star\prec \bm{n}^{\star\star} \implies  p(\bm{n}^\star) = p(\bm{n}^{\star\star})$. We have $p(\bm{n})=p(\bm{n}')$ for all $\bm{n},\bm{n}'\in \mathcal{I}_n^k$.  Thus, EPPF $p$ only depends on $n$ and $k$ and hence can be written as $p(n_1,\ldots,n_k) \stackrel{let}{=} v(n,k)$. In other words, $p$ is a Gibbs partition with $V_{n,k}=v(n,k)$ and $W_s \equiv 1$. \citet{gnedin2006exchangeable} showed that \textit{any} infinitely exchangeable Gibbs partition with $\sigma<0$, including $\sigma=-\infty$, can be expressed as a mixture of Ewens-Pitman two parameter family $(\sigma, |\sigma| K)$ with some mixing distribution $q$ on the number of components $K\in\mathbb{N}$, thus the theorem follows.


\subsection{Proof of Theorem 4.1}
By \cref{thm:gibbs}, to prove the equivalency of balance-averseness and $\bm\mu \gelc \mathsf{Poisson}^+$, it is sufficient to prove that $\bm\mu \gelc \mathsf{Poisson}^+$ if and only if the corresponding $\bm{W}$ is log-convex. Rewrite the EPPF of ESC model as:
\begin{equation}
    p^{(n)}(n_1,\ldots,n_k|\bm{\mu}) = \frac{\mu_1^k}{\mathbb{P}(E_n|\bm{\mu})}\frac{k!}{n!}\prod_{j=1}^k n_j!\mu_{n_j}/\mu_1 = V_{n,k}\prod_{j=1}^kW_{n_j}
    \label{eq:esc}
\end{equation}
The log-convexity of $(W_s)_{s=1}^\infty$ is equivalent to the convexity of $(\log (s!\mu_s))_{s=1}^\infty$, or concavity of $(-\log (s!\mu_s)+a s+b)_{s=1}^\infty$ for arbitrary constants $a$ and $b$. The log-convexity of $(s!\mu_s)_{s=1}^\infty$, in other words $s\mu^2_s\le(s+1)\mu_{s-1}\mu_{s+1}$, implies $\mu_s>0$ for all $s=1,2,\ldots$ (full support), since assuming $\mu_s=0$ for some $s=s^*$ leads to $\mu_s=0$ for $s\le s^*$ that contradicts with $\mu_1>0$. By choosing $a=\log \lambda$ and $b=\log (e^{-\lambda}/(1-e^{-\lambda}))$ for some $\lambda>0$, we can see that concavity of $(-\log (s!\mu_s)+a s+b)_{s=1}^\infty$ is equivalent to $\bm\mu \gelc \mathsf{Poisson}^+(\lambda)$, by the definition of relative log-concavity and $\mu_s>0$ for all $s$ (full support).
Next, to show that $\bm\mu \lelc \mathsf{Poisson}^+$ if and only if the corresponding $\bm{W}$ is log-concave, we can apply the similar argument as above by flipping the inequalities. 
Now we prove the balance-neutral case, which amounts to showing that $\bm\mu = \mathsf{Poisson}^+$ is the only discrete probability measure such that $\bm{W} = (s!\mu_s/\mu_1)_{s=1}^\infty$ is both log-convex and log-concave. In other words, if $\log s! \mu_s$ is a linear sequence of $s$, then $\bm\mu = \mathsf{Poisson}^+$. Let $\log s! \mu_s = as + b$ for some constants $a$ and $b$, then $\mu_s = (e^{as}e^b)/s!$ for $s=1,2,\cdots$. Solving $\sum_{s=1}^\infty (e^{as}e^b)/s! = 1$ leads to $b=-\log (\exp(e^a)-1)$. Now substituting $\lambda = e^a>0$, we have $\mu_s  = \frac{\lambda^s}{s!}\frac{1}{e^\lambda-1}$, a probability mass function of $\mathsf{Poisson}^+(\lambda)$.

\section{Details on B-sequence}
\label{sec:appendixb}
\subsection{Slope Interpretation of B-sequence}
This subsection provides the derivation of \cref{eq:slope}. Let $\bm{n}=(n_1,\ldots,n_k),\bm{n}'=(n_1',\ldots,n_k')\in\mathcal{I}_n^k$ where $\bm{n}'$ is a $(**)-$cover of $\bm{n}$. Then there exists some $s\ge 2$ such that 
\[
\bm{n} = (n_1,\cdots n_u,\cdots, n_v,\cdots, n_k) \mapsto (n_1,\cdots, \underbrace{n_u-1}_{=s}, \cdots, \underbrace{n_v+1}_{=s}, \cdots, n_k) = \bm{n}'
\]
where $s=n_u'=n_u-1=n_v'=n_v+1$ with $u<v$. By \cref{eq:gibbspartition}, since $\bm{n}$ and $\bm{n}'$ both have $k$ clusters, 
\begin{align*}
    \log p^{(n)}(\bm{n}')-\log p^{(n)}(\bm{n})&= \log W_{n_u'}+\log W_{n_v'} - \log W_{n_u} - \log W_{n_v}= -\log W_{s+1} + 2\log W_s - \log W_{s-1}
\end{align*}
Also, for the Shannon index $H(\bm{n}) = -\sum_{j=1}^k(n_j/n)\log (n_j/n)$,
\begin{align}
    H(\bm{n}') - H(\bm{n}) &=\left(-\frac{n_u'}{n}\log \frac{n_u'}{n}-\frac{n_v'}{n}\log \frac{n_v'}{n}\right) - \left(-\frac{n_u}{n}\log \frac{n_u}{n}-\frac{n_v}{n}\log \frac{n_v}{n}\right)\\
    &= \frac{1}{n}\left[ (s+1)\log \frac{s+1}{n} - 2s\log \frac{s}{n} + (s-1)\log \frac{s-1}{n}\right] \label{eq:approxbefore}\\
    &\approx\footnotemark \frac{1}{n} \left[ \frac{d^2}{dx^2}\left(x\log\frac{x}{n}\right)\right]_{x=s}= \frac{1}{n}\frac{1}{s} \label{eq:approxafter}
\end{align}
\footnotetext{The approximation error is about $0.023/n$ for $s=2$, $0.007/n$ for $s=3$, $0.003/n$ for $s=4$, $0.001/n$ for $s=5$, and less than $0.001/n$ for $s\ge 6$.}

Therefore 
\begin{equation}
    \frac{\log p^{(n)}(\bm{n}')-\log p^{(n)}(\bm{n})}{H(\bm{n}') - H(\bm{n})} \approx n\times -s(\log W_{s+1}-2\log W_s + \log W_{s-1}) = nB_s(\bm{W})
\end{equation}
and see \cref{fig:balanceindexinterpret} for the grapical illustration.

\subsection{B-sequence and Balancedness}
First, we prove that the ordering based on the B-sequence comparison is equivalent to the log-concavity ordering. 

\begin{theorem}\label{thm:Bseq1}
$\bm{W}\lelc\bm{W}'$ if and only if $B_s(\bm{W})\ge B_s(\bm{W}')$ for all $s\ge 2$.
\end{theorem}

\begin{proof} 
Below, we prove the cases when $\bm{W}$ have a full support or a finite support, respectively. 
Assume that $\bm{W}$ have a full support; i.e. $W_s>0$ for all $s=1,2\ldots$. Then $\bm{W}\lelc\bm{W}'$ implies $\bm{W}'$ also have a full support by its definition. Note that $ B_s(\bm{W})\ge B_s(\bm{W}'), \forall s$ also implies $B_s(\bm{W}')<\infty$ for all $s=2,3,\ldots$, and thus $\bm{W}'$ have a full support. Therefore, either assuming $\bm{W}\lelc\bm{W}'$ or $B_s(\bm{W})\ge B_s(\bm{W}')$ for $s=2,3,\ldots,$ we have $\mathrm{supp}(\bm{W})\subseteq \mathrm{supp}(\bm{W}')$. Then,
\begin{align*}
    \bm{W}\lelc\bm{W}' \iff & (\log(W_s/W_{s}'))\text{ is concave}\\
    \iff & \left(\frac{W_s}{W_s'}\right)^2\ge \frac{W_{s+1}W_{s-1}}{W_{s+1}'W_{s-1}'}, \quad \forall s=2,3,\ldots\\
    \iff&\frac{W_{s+1}'W_{s-1}'}{(W_{s}')^2}\ge  \frac{W_{s+1}W_{s-1}}{(W_s)^2}, \quad \forall s=2,3,\ldots\\
    \iff& \log W_{s-1}'-2\log W_{s}'+ \log W_{s+1}' \ge \log W_{s-1}-2\log W_{s}+ \log W_{s+1}, \quad \forall s=2,3,\ldots\\
    \iff& B_s(\bm{W})\ge B_s(\bm{W}'), \quad \forall s=2,3,\ldots
\end{align*}
Next, assume that $\bm{W}$ has a finite support $\{1,2,\ldots,m\}$, which only happens when $\bm{W}$ is log-concave. Then by $\bm{W}\lelc \bm{W}'$, $\bm{W}'$ must have either a finite or a full support, and if it has a finite support $\{1,2,\ldots,m'\}$, then $m'\ge m$. On the other hand, $B_s(\bm{W})\ge B_s(\bm{W}')$ for $s=2,3,\ldots$ implies that the support of $\bm{W}$ must be contained in the support of $\bm{W}'$, since otherwise there is an $s$ such that $B_s(\bm{W}')=+\infty$ which contradicts with $B_s(\bm{W})<\infty$. Thus, either assuming $\bm{W}\lelc\bm{W}'$ or $B_s(\bm{W})\ge B_s(\bm{W}')$ for $s=2,3,\ldots,$ we have $\mathrm{supp}(\bm{W})\subseteq \mathrm{supp}(\bm{W}')$. Then,
\begin{align*}
 \bm{W}\lelc\bm{W}' \iff & (\log(W_s/W_{s}'))\text{ is concave in }\mathrm{supp}(\bm{W}) \\
    \iff & \left(\frac{W_s}{W_s'}\right)^2\ge \frac{W_{s+1}W_{s-1}}{W_{s+1}'W_{s-1}'}, \quad \forall s=2,3,\ldots,m-1\\
    \iff&\frac{W_{s+1}'W_{s-1}'}{(W_{s}')^2}\ge  \frac{W_{s+1}W_{s-1}}{(W_s)^2}, \quad \forall s=2,3,\ldots,m-1\\
    \iff& \log W_{s-1}'-2\log W_{s}'+ \log W_{s+1}' \ge \log W_{s-1}-2\log W_{s}+ \log W_{s+1}, \quad \forall s=2,3,\ldots,m-1\\
    \iff& B_s(\bm{W})\ge B_s(\bm{W}'), \quad \forall s=2,3,\ldots,m-1
\end{align*}
which proves the claim.
\end{proof}

By \cref{thm:Bseq1} and \cref{thm:gibbs}, we have the following corollary:
\begin{corollary} For any $n=1,2,\ldots,$ $\mathsf{Gibbs}_{[n]}(\bm{V},\bm{W})$ is 
\begin{itemize}
    \item balance-averse if and only if $B_s(\bm{W})\le 0$ for all $s\ge 2$
    \item balance-neutral if and only if $B_s(\bm{W})= 0$ for all $s\ge 2$ 
    \item balance-seeking if and only if $B_s(\bm{W})\ge 0$ for all $s\ge 2$
\end{itemize}
\end{corollary}

To see this, $B_s(\bm{W}) \equiv 0$ if and only if $\bm{W}$ has a full support and $W_s^2=W_{s+1}W_{s-1}$ for all $s\geq 2$, which is equivalent to $(W_s)$ being log-linear (both log-convex and log-concave). Similar argument follows for the balance-averse and the balance-seeking cases. 


\section{Details on ESC Model and Posterior Inference}
\label{sec:appendixc}
\subsection{Examples of ESC Model with Different Balancedness}

In this section, we provide details of ESC models with different balancedness properties. We consider 5 different cases, where $\bm\mu=(\mu_s)_{s=1}^\infty$ is a probability mass function of 1) shifted binomial, 2) zero-truncated binomial, 3) zero-truncated Poisson, 4) zero-truncated (extended) negative binomial, and 5) logarithmic distribution. We assume the number of trials $N$ is fixed in zero-truncated binomial, while shifted binomial is not. Recall that EPPF of ESC model with fixed $\bm\mu$ and $\mu_1>0$ is written as follows~\citep{betancourt2020random}:
\begin{equation}
    p^{(n)}(n_1,\ldots,n_k|\bm{\mu}) = \frac{1}{\mathbb{P}(E_n|\bm{\mu})}\frac{k!}{n!}\prod_{j=1}^k n_j!\mu_{n_j} 
\end{equation}
where the event $E_n$ is $E_n = \left\{\text{there exists }k\in\mathbb{N}\text{ such that }\sum_{j=1}^k n_j = n\right\}$.  We can calculate the normalizing constant $\mathbb{P}(E_n|\bm{\mu})$ with the $k$-fold convolution of random variables following the distribution $\bm\mu$. That is, if $X_1,\ldots,X_k\stackrel{iid}{\sim}\bm\mu$ and $S_k := X_1+\cdots+X_k$, then $\mathbb{P}(E_n|\bm\mu) = \sum_{k=1}^n\mathbb{P}(S_k = n)$, which is always positive if $\mu_1>0$. We provide five examples of $\bm\mu$ in the ESC model in Table \ref{table:tableofab1class}, with their complete form of EPPFs including the normalizing constant. For the derivation of $k$-fold convolution formulas, we refer the readers to \citet{ahuja1970distribution,tate1958minimum,ahuja1971distribution,patil1965certain}.
Regarding the notations, $S_1(n,k)$ and $S_2(n,k)$ are the Stirling number of the first and second kind respectively, and the binomial coefficient with real argument is defined as $\binom{a}{b} = \frac{a(a-1)\cdots(a-b+1)}{b!}$ for $a\in \mathbb{R}$ and $b\in\mathbb{N}$. When $r=1$ in $\mathsf{NegBin}^+(r,p)$, it becomes $\mathsf{Geometric}(p)$ distribution and the normalizing constant $\mathbb{P}(E_n|\bm\mu)$ is simply $p$. We note that $\mathsf{NegBin}^+(r,p)$ with $r>-1, r\neq 0$ is also called Engen’s extended negative binomial distribution since it admits negative $r$ values.

\begin{table}[h]
\small
\caption{Examples of $\bm\mu$ in the ESC model. ESC models with shifted binomial and zero-truncated binomial are balance-seeking, zero-truncated Poisson is balance-neutral, and zero-truncated negative binomial and logarithmic are balance-averse.}
\label{table:tableofab1class}
\vskip 0.1in
\renewcommand{\arraystretch}{1.2}
\begin{tabular}{c c c c}
\toprule
$\bm\mu$ & Support  & $\mu_{s}$ & $\mathbb{P}(E_n|\bm\mu)$\\ \hline 
$\begin{array}{c}1+\mathsf{Bin}(N,p);\\ N \in \mathbb{N}\\ p\in(0,1)\end{array}$ & $\{1,\!\cdots\!,N\!+\!1\}$ &   $\displaystyle \frac{N!p^{s-1}(1-p)^{N-s+1}}{(s-1)!(N-s+1)!}$ & $\displaystyle \sum_{k=\lceil n/(N+1)\rceil}^n  \binom{Nk}{n-k}p^{n-k}(1-p)^{Nk-n+k}
$\\ \hline
$\begin{array}{c}\mathsf{Bin}^+(N,p);\\ N\in\mathbb{N} \text{ (fixed)}\\ p\in(0,1)\end{array}$ & $\{1,\cdots,N\}$ &  $\displaystyle \binom{N}{s}\frac{p^s(1-p)^{N-s}}{1-(1-p)^N}$ & $\displaystyle \sum_{k=\lceil n/N \rceil}^n\sum_{i=1}^{k} \left\{\binom{k}{i}\binom{Ni}{n}\frac{(-1)^{k-i}p^n(1-p)^{Nk-n}}{(1-(1-p)^N)^k}\right\}
$\\ \hline
$\begin{array}{c}\mathsf{Poisson}^+(\lambda);\\ \lambda>0\end{array}$ & $\mathbb{N}$& $\displaystyle \frac{\lambda^se^{-\lambda}}{(1-e^{-\lambda})s!}$ & $\displaystyle\sum_{k=1}^n\frac{k!\lambda^n}{(e^\lambda -1)^k n!}S_2(n,k)
$\\ \hline 
$\begin{array}{c}\mathsf{NegBin}^+(r,p);\\ r>-1, r\neq 0\\ p\in(0,1)\end{array}$& $\mathbb{N}$&  $\displaystyle  \binom{s+r-1}{s}\frac{(1-p)^rp^s}{1-(1-p)^r}$ & $\displaystyle \sum_{k=1}^n \sum_{i=1}^k\left\{ \binom{k}{i}\binom{n+ri-1}{n}\frac{(-1)^{k-i}p^n}{((1-p)^{-r}-1)^k}\right\}
$\\ \hline
$\begin{array}{l}\mathsf{Logarithmic}(p);\\ p\in(0,1)\end{array}$ & $\mathbb{N}$&  $\displaystyle  \frac{-1}{\log(1-p)}\frac{p^s}{s}$ & $\displaystyle \sum_{k=1}^n \frac{k!p^n}{n!(-\log(1-p))^k}|S_1(n,k)|$\\ 
\bottomrule
\end{tabular}
\end{table}

\subsection{Posterior Inference Algorithms}

We provide posterior inference algorithms for the ESC models presented in the previous subsection. Let $\bm{x}$ be data, $\theta_\mu$ be the parameter(s) of $\bm\mu$, $\bm\beta=(\beta_\ell)$ be the distortion probabilities, and $\mathbf{z} = (z_1,\ldots,z_n)$ be the cluster membership vector. Categorical distribution probabilities $\bm\theta_\ell$ are fixed and assumed to be the empirical distribution of the data, following \citet{steorts2015entity}. Each iteration of MCMC sampler consists of three steps: first update $\Pi_n | \theta_\mu, \bm\beta,\bm{x}$, next update $\theta_\mu|\Pi_n, \bm{x}$, and then update $\bm\beta | \Pi_n, \theta_\mu, \bm{x}$. We provide details of the first two which are affected by the modification of random partition prior; the likelihood function and updating scheme of $\bm\beta$ are similar to \citet{steorts2015entity,betancourt2020random} and hence omitted.

1. Update $\Pi_n | \theta_\mu, \bm\beta, \bm{x}$ using the following full conditional distribution:
    \begin{equation*}
    \mathbb{P}(z_{i}=j | -) \propto p(\bm{x}|\mathbf{z}_{-i}, z_{i}=j, \bm\beta) \mathbb{P}(z_i=j|\mathbf{z}_{-i}, \theta_\mu),
    \end{equation*}
    where the prior reallocation probability $\mathbb{P}(z_i=j|\mathbf{z}_{-i}, \theta_\mu)$ is
\begin{itemize}
    \item (Shifted binomial) \quad $\mathbb{P}(z_i=j|\mathbf{z}_{-i}, N,p)\propto\begin{cases}
    \frac{n_j^{-i}+1}{n_j^{-i}}(N-n_j^{-i}+1) &\text { if } j=1, \ldots, k^{-i}, \\
    \left(k^{-i}+1\right) (1-p)^{N+1}/p & \text { if } j=k^{-i}+1
    \end{cases}$
    \item (Zero-truncated binomial) \quad $\mathbb{P}(z_i=j|\mathbf{z}_{-i}, N,p)\propto \begin{cases}
    (N-n_j^{-i}) & \text { if } j=1, \ldots, k^{-i}, \\
\left(k^{-i}+1\right) \frac{N(1-p)^N}{1-(1-p)^N} & \text { if } j=k^{-i}+1
\end{cases}$
    \item (Zero-truncated Poisson) \quad $ \mathbb{P}(z_i=j|\mathbf{z}_{-i}, \lambda) \propto \begin{cases}
    1 & \text { if } j=1, \ldots, k^{-i}, \\
\left(k^{-i}+1\right) \frac{e^{-\lambda}}{1-e^{-\lambda}} & \text { if } j=k^{-i}+1
\end{cases}$
    \item (Zero-truncated negative binomial) \quad  $\mathbb{P}(z_i=j|\mathbf{z}_{-i}, r,p)\propto \begin{cases}
    n_j^{-i}+r & \text { if } j=1, \ldots, k^{-i}, \\
    \left(k^{-i}+1\right) \frac{r(1-p)^r}{1-(1-p)^r} & \text { if } j=k^{-i}+1
    \end{cases}$
\item (Logarithmic) \quad $ \mathbb{P}(z_i=j|\mathbf{z}_{-i}, p)\propto \begin{cases}
     n_j^{-i} & \text { if } j=1, \ldots, k^{-i}, \\
     \left(k^{-i}+1\right)\frac{-1}{\log(1-p)}  & \text { if } j=k^{-i}+1
    \end{cases}$
\end{itemize}
where $k^{-i}$ and $n_j^{-i}$ are the number and sizes of clusters in $\mathbf{z}_{-i}$. Notice that in the last four examples, the reallocation probability to an existing cluster is a linear function of its size, and these are the only possible cases of $\bm\mu$ that have such a property \citep{johnson2005univariate}. These are called Sundt and Jewell family \citep{sundt1981further} or (a,b,1) zero-truncated family \citep{klugman2012loss} in actuarial applications.

2. Update $\theta_\mu|\Pi_n,\bm{x}$, for example using slice sampler~\citep{neal2003slice}, 
\begin{itemize}
    \item (Shifted binomial) Here $\theta_\mu = (N,p)$ and let prior $\pi(N,p)\propto p^{-0.5}(1-p)^{-0.5}/N$ following \citet{berger2012objective}. Then unlike the zero-truncated binomial, we can use the following joint distribution sampler: first update $N$ from $[N|\Pi_n,\bm{x}]$,
    \begin{equation*}
      [N|\Pi_n,\bm{x}] \propto  \frac{\mathcal{B}(n-k+\frac{1}{2},Nk-n+k+\frac{1}{2})}{N}\prod_{j=1}^k\frac{n_j N!}{(N-n_j+1)!}  , \quad N=\max n_j - 1, \max n_j, \cdots
    \end{equation*}
    where $\mathcal{B}$ is a beta function. Given this $N$, update $p|N,\Pi_n,\bm{x}\sim \mathsf{Beta}(n-k+\frac{1}{2}, Nk-n+k+\frac{1}{2})$.
    \item (Zero-truncated binomial with fixed $N$) Here $\theta_\mu=p$ and let prior $p\sim \mathsf{Beta}(a_p,b_p)$. Then sample $p$ from 
    \begin{equation}
      [p|\Pi_n,\bm{x}] \propto  p^{a_p-1}(1-p)^{b_p-1}\frac{p^n(1-p)^{Nk-n}}{(1-(1-p)^N)^k}
    \end{equation}
    \item (Zero-truncated Poisson) Here $\theta_\mu=\lambda$ and let prior $\lambda\sim \mathsf{Gamma}(a_\lambda,b_\lambda)$. Then sample $\lambda$ from 
    \begin{equation}
      [\lambda|\Pi_n,\bm{x}] \propto  \lambda^{a_\lambda-1}\exp(-b_\lambda\lambda)  \lambda^n \left(\frac{e^{-\lambda}}{1-e^{-\lambda}}\right)^k
    \end{equation}
    \item (Zero-truncated negative binomial) Here $\theta_\mu=(r,p)$ and let prior $(r,p)\sim \pi(r,p)$. Then sample $(r,p)$ from 
    \begin{equation}
      [r,p|\Pi_n,\bm{x}] \propto \pi(r,p) p^n\left(\frac{(1-p)^r}{1-(1-p)^r}\right)^k \prod_{j=1}^k[(n_j+r-1)\cdots (r)]  
    \end{equation}
        \item (Logarithmic) Here $\theta_\mu=p$ and let prior $p\sim \mathsf{Beta}(a_p,b_p)$. Then sample $p$ from 
    \begin{equation}
      [p|\Pi_n,\bm{x}] \propto  p^{a_p-1}(1-p)^{b_p-1}p^n\left(\frac{-1}{\log(1-p)}\right)^k
    \end{equation}
\end{itemize}

The software for the described posterior inference algorithms will be available in R package $\mathtt{microclustr}$ \citep{microclustr}.

\section{Details on Real and Synthetic Data Studies}
\label{sec:appendixd}

\subsection{SIPP1000 Data Analysis Details}

In this section, we provide more details on SIPP1000 data analysis described in \ref{sec:4.3}. SIPP1000 dataset has $n=4116$ records, and it contains $5$ categorical fields (sex, birth year, birth month, race, and state of residence) of $K^+=1000$ individuals. We fit ESC models with four different choices of $\bm\mu$: i) $\mathsf{Binomial}^+$ with fixed number of trials $N=5$ (balance-seeking), ii) $\mathsf{Poisson}^+$ (balance-neutral), iii) $\mathsf{NegativeBinomial}^+$ (balance-averse), and iv) $\mathsf{Dirichlet}$ \citep[neither balance-seeking nor averse;][]{betancourt2020random}. We fix the number of trials as $N=5$ in the binomial case since dataset is constructed from five waves of longitudinal study, so there is no cluster with size greater than 5 a priori. The hyperprior specifications are i) $p\sim \mathsf{Beta}(0.5,0.5)$ for the binomial success probability, ii) $\lambda\sim \mathsf{Gamma}(1,1)$ for $\mathsf{Poisson}^+$, iii) $(r,p)\sim \mathsf{Gamma}(1,1)\times\mathsf{Beta}(2,2)$ for $\mathsf{NegBin}^+$, and iv) $\alpha=1$, $\bm\mu^{(0)}|r,p=\mathsf{NegBin}^+(r,p)$, and $(r,p)\sim \mathsf{Gamma}(1,1)\times\mathsf{Beta}(2,2)$ for $\mathsf{Dirichlet}$ (see \citet{betancourt2020random} for details). Prior for distortion parameters $(\beta_\ell)$ are assumed to be independent beta distributions with mean 0.005 and standard deviation 0.01. 

We collect 15000 posterior samples after 5000 burn-in iterations, where we update cluster indicators $(z_i)$ for each individual within each one (global) MCMC iteration. All computations were performed on an Intel E5-2690 v3 CPU with 128GB of memory. 
In \cref{fig:sippmcmck} and \cref{fig:sippmcmcparameters}, present trace plots for the number of entities $K^+$ as well as parameters $\theta_\mu$, where four chains all become stationary after 5000 burn-in iterations.

\begin{figure}[h]
    \centering
    \includegraphics[width=0.99\textwidth]{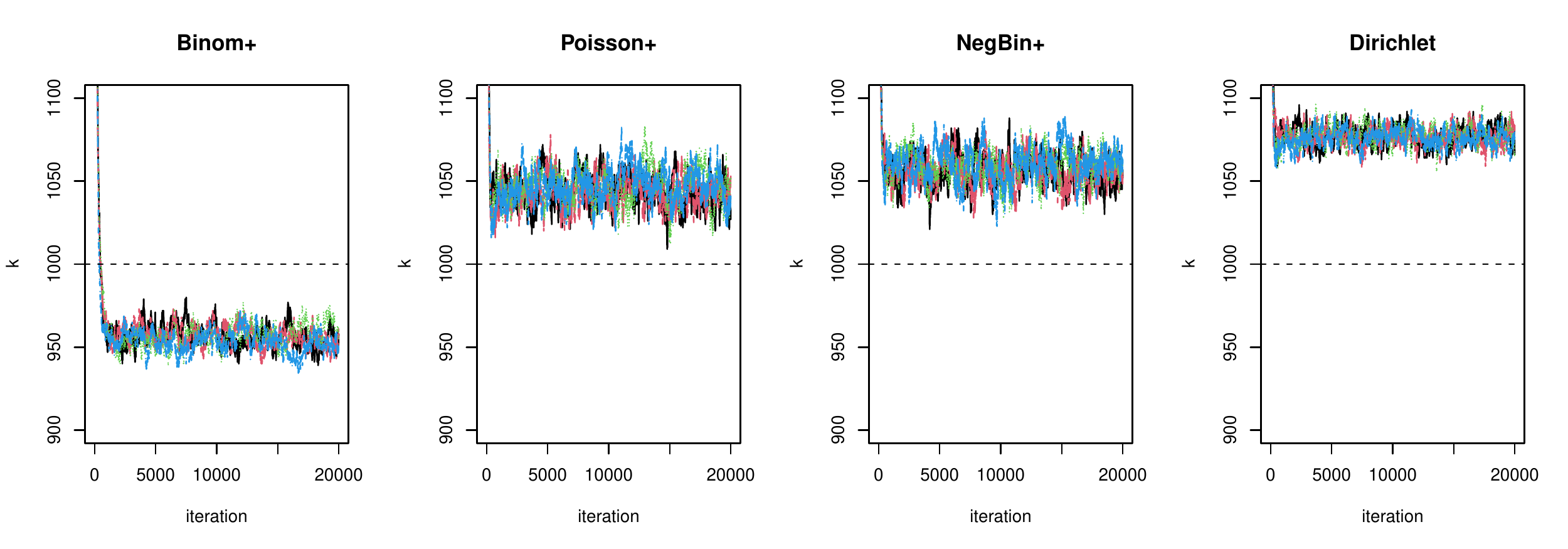}
    \caption{SIPP1000 dataset; trace plots of the number of clusters (entities) $K^+$ where true is 1000. }
    \label{fig:sippmcmck}
\end{figure}

\begin{figure}[h]
    \centering
    \includegraphics[width=0.99\textwidth]{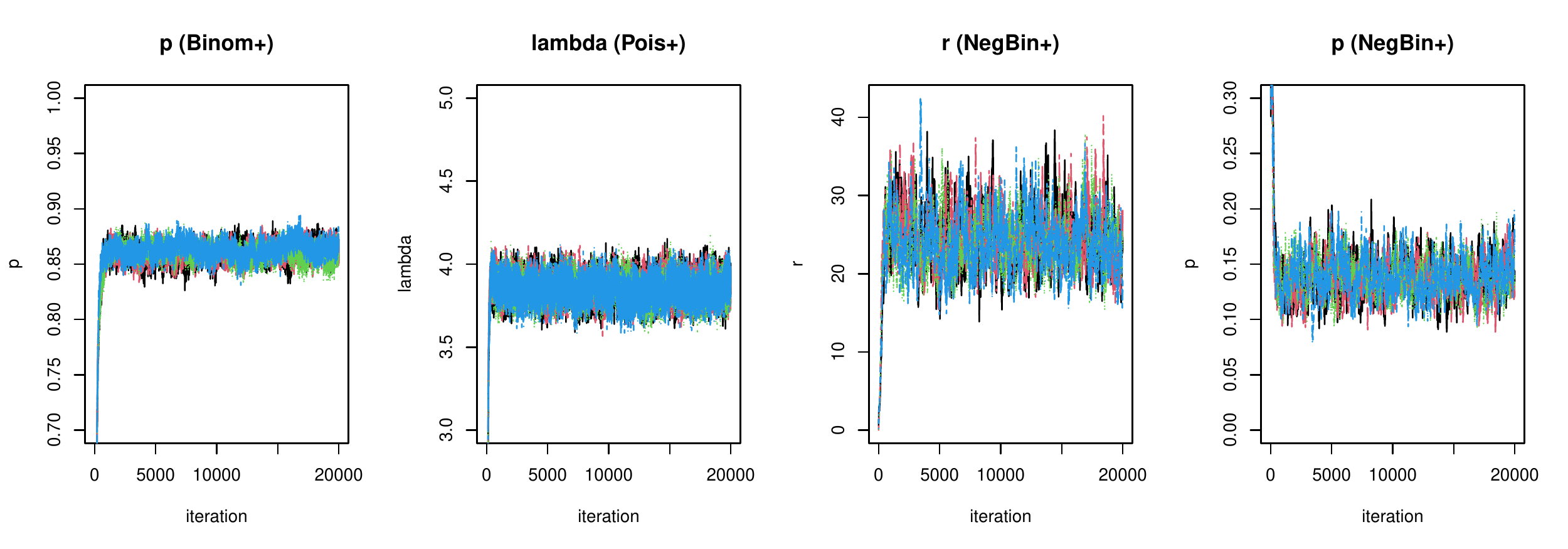}
    \caption{SIPP1000 dataset; trace plots of random partition model-specific parameters $\theta_\mu$. }
    \label{fig:sippmcmcparameters}
\end{figure}

In addition to \cref{table:SIPP}, we also report the summary of posterior distributions of distortion probabilities $(\beta_\ell)$ of SIPP1000 dataset. The distortion probability of the month of birth, $\beta_3$, is the highest for all models. Balance-seeking model ($\mathsf{Binom}^+$) provides higher estimated distortion probabilities compared to other models, even though the prior for $(\beta_\ell)$ are all the same. This is because the balance-seeking model avoids the generation of singleton clusters and prefers to treat the difference as `distorted entry'. This leads to the reduced false negative rate (FNR) by reducing the number of missed pairs (i.e. record pairs that are linked under the truth but not linked under the estimate), but at the cost of increased FDR.

\begin{table}[th]
\caption{SIPP1000 dataset; Summary of posterior distributions (mean, standard deviation) of  distortion probabilities $(\beta_l)$ for different $\bm\mu$. $\beta_1,\ldots,\beta_5$ correspond to distortion probabilities of sex, year of birth, month of birth, race, and state of residence, respectively.}
\label{table:SIPPbeta}
\vskip 0.1in
\begin{center}
\begin{small}
\begin{tabular}{c cccccccccc}
\toprule
$\bm{\mu}$ (balance-) & $\mathbb{E}(\beta_1|\bm{x})$& SD & $\mathbb{E}(\beta_2|\bm{x})$& SD & $\mathbb{E}(\beta_3|\bm{x})$& SD & $\mathbb{E}(\beta_4|\bm{x})$& SD & $\mathbb{E}(\beta_5|\bm{x})$& SD  \\
\hline
$\mathsf{Binom}^+$ (seeking) & 0.0012 & 0.0014 & 0.0151 & 0.0027 & 0.0492 & 0.0039 & 0.0074 & 0.0029 & 0.0140 & 0.0025\\
$\mathsf{Poisson}^+$ (neutral) & 0.0003 & 0.0005 & 0.0032 & 0.0018& 0.0401 & 0.0035 & 0.0008 & 0.0012 & 0.0039 & 0.0019\\
$\mathsf{NegBin^+}$ (averse) & 0.0003 & 0.0006 & 0.0026 & 0.0019 & 0.0391 & 0.0036 & 0.0007 & 0.0010 & 0.0025 & 0.0019\\
$\mathsf{Dirichlet}$ (n/a) & 0.0002 & 0.0004 & 0.0002 & 0.0004 & 0.0377 & 0.0034 & 0.0016 & 0.0016 & 0.0009 & 0.0008\\ %
\hline
\end{tabular}
\end{small}
\end{center}
\vskip -0.1in
\end{table}

\newpage 

\subsection{Simulation Studies}

We conduct additional simulations to study the effect of balancedness on entity resolution tasks. We consider three different scenarios of cluster size distribution: 1) binomial, 2) Poisson, and 3) negative binomial, where \cref{table:simulationsettings} shows the cluster size distributions for each scenario. We simulate data with the generating scheme described in \cref{sec:4.1}, where $L=5$, $D_\ell = 10$, $\bm\theta_\ell = (0.1,\ldots,0.1)$, and three different choices of distortion probabilities $\beta\in\{0.01, 0.05, 0.1\}$. 

\begin{table}[th]
\caption{The true number of clusters $m_s=\sum_{j}1(n_j=s)$ of each size $s=1,\ldots,14$ for three different scenarios. Each entries correspond to 100 times the probability mass of the distribution shown in the first column, rounded to the nearest integer. }
\label{table:simulationsettings}
\vskip 0.1in
\begin{center}
\begin{small}
\begin{tabular}{c ccccc ccccc cccc cc}
\hline
Scenario & $m_1$ & $m_2$ & $m_3$ & $m_4$ & $m_5$ & $m_6$& $m_7$& $m_8$& $m_9$& $m_{10}$& $m_{11}$& $m_{12}$ & $m_{13}$& $m_{14}$ & $K$\\
\hline
1. $\mathsf{Binom}^+(10,0.5)$ & 1 & 4 & 12 & 21 & 25 & 21 & 12 & 4 & 1 & 0 & 0 & 0 & 0 & 0 & 101\\
2. $\mathsf{Poisson}^+(5)$ & 3 & 8 & 14 & 18 & 18 & 15 & 11 & 7 & 4 & 2 & 1 & 0 & 0 & 0 & 101 \\
3. $\mathsf{NegBin^+}(5,0.5)$ & 8 & 12 & 14 & 14 & 13 & 11 & 8 & 6 & 5 & 3 & 2 & 1 & 1 & 1 & 99 \\
\hline
\end{tabular}
\end{small}
\end{center}
\vskip -0.1in
\end{table}

We fit the ESC model with four different $\bm\mu$: i) $1+\mathsf{Binom}$ (shifted binomial), ii) $\mathsf{Poisson}^+$, iii) $\mathsf{NegBin}^+$, and iv) $\mathsf{Dirichlet}$. We use shifted binomial instead of zero-truncated binomial, assuming the situation when the maximum size of clusters is not known a priori. For shifted binomial distribution, we put $p(N,p)\propto 1/N\times p^{-0.5}(1-p)^{-0.5}$ prior on $(N,p)$ following \citet{berger2012objective}, and other hyperparameter settings for other three models and MCMC specification details are same as the previous subsection in the SIPP1000 data analysis. Here distortion probabilities $(\beta_\ell)$ are assumed to be unknown.

\begin{table}[!h]
\caption{Simulation study results for three different scenarios and three different true distortion probabilities. Presented are the posterior mean and standard deviation of the number of clusters, and the point estimates of FNRs and FDRs in \% by \citet{dahl2006model}'s method.}
\label{table:simsetting123}
\vskip 0.1in
\begin{center}
\begin{small}
\begin{tabular}{c|c|ccc|ccc|ccc}
\toprule
& & \multicolumn{3}{c|}{Scenario 1 (true $K^+=101$)} & \multicolumn{3}{c|}{Scenario 2 (true $K^+=101$) } & \multicolumn{3}{c}{Scenario 3 (true $K^+=99$)}\\
$\beta$ & $\bm{\mu}$ (balance-) & $\mathbb{E}(K^+|\bm{x})$ (SD) & FNR & FDR & $\mathbb{E}(K^+|\bm{x})$ (SD) & FNR & FDR & $\mathbb{E}(K^+|\bm{x})$ (SD) & FNR & FDR \\
\midrule
\multirow{4}{*}{0.01}& $1+\mathsf{Binom}$ (seeking) & 100.8 (0.6) & \textbf{0.0} & \textbf{0.0} &  100.6 (1.1) & \textbf{0.9} & 0.8 & 98.9 (1.3) & 0.6 & 0.1 \\
& $\mathsf{Poisson}^+$ (neutral) & 101.3 (0.7) & 0.4 & 0.5 & 101.5 (1.1) & 1.1 & \textbf{0.6} & 100.1 (1.1) & 0.6 & 0.1 \\
& $\mathsf{NegBin^+}$ (averse)  & 101.6 (0.9) & 0.4 & 0.5 & 102.2 (1.2) & 1.1 & \textbf{0.6} & 101.3 (1.4) & 0.6 & 0.1  \\
& $\mathsf{Dirichlet}$ (n/a) & 101.7 (0.9) & 0.7 & \textbf{0.0} & 102.1 (1.5) & 1.5 & 0.8 & 103.6 (2.0) & 0.6 & 0.1\\ 
\midrule
\multirow{4}{*}{0.05}& $1+\mathsf{Binom}$ (seeking) &  99.7 (1.2) & \textbf{2.9} & 3.2 & 100.0 (1.5) & \textbf{2.4} & 2.9 & 95.2 (1.5) & \textbf{3.0} & 3.5\\
& $\mathsf{Poisson}^+$ (neutral) & 103.8 (2.1) & 4.2 & \textbf{2.4} & 102.4 (2.0) & 2.9 & 2.6 & 96.4 (1.7) & 3.6 & 4.4 \\
& $\mathsf{NegBin^+}$ (averse)  & 106.4 (2.2) & 6.7 & 3.6 & 105.2 (2.4) & 3.2 & 2.6 & 99.4 (2.3) & 3.1 & \textbf{2.8} \\
& $\mathsf{Dirichlet}$ (n/a) & 102.1 (2.6) & 4.8 & 4.1 & 103.6 (3.3) & 3.8 & \textbf{2.0} & 100.6 (3.0) & 4.9 & 4.0\\ 
\midrule
\multirow{4}{*}{0.1}& $1+\mathsf{Binom}$ (seeking) & 105.2 (1.9) & \textbf{8.1} & 4.1 & 98.5 (2.2) & \textbf{7.0} & 9.4  & 94.0 (2.2) & 8.1 & 7.5\\
& $\mathsf{Poisson}^+$ (neutral) & 111.1 (2.9) & 9.2 & 4.1 & 102.9 (2.9) & 8.9 & 10.1 & 98.7 (3.0) & \textbf{7.3} & 8.1\\
& $\mathsf{NegBin^+}$ (averse)  & 115.6 (3.6) & 10.5 & \textbf{3.9} & 109.5 (4.0) & 10.7 & 8.8 & 109.1 (4.3) & 10.9 & \textbf{5.4} \\
& $\mathsf{Dirichlet}$ (n/a) & 109.9 (5.0) & 8.8 & 4.1 & 113.9 (5.5) & 12.9 & \textbf{6.5} & 112.0 (5.0) & 12.7 & 6.2\\ 
\bottomrule
\end{tabular}
\end{small}
\end{center}
\vskip -0.1in
\end{table}

\cref{table:simsetting123} shows that the balance-seeking random partition model generally achieves a lower FNR compared to others, even under Scenarios 2 and 3 where the model is misspecified. \cref{fig:figsim123} shows the posterior distributions of $m_i$, the number of clusters of size $i$. It can be clearly seen that the balance-seeking model regularizes the generation of singleton clusters, suggesting its usefulness for ER tasks when a single record per individual is less likely to happen or when false negative control is more important.

\begin{figure}[!ht]
    \centering
    \includegraphics[height=1.05\textwidth,width=0.93\textwidth]{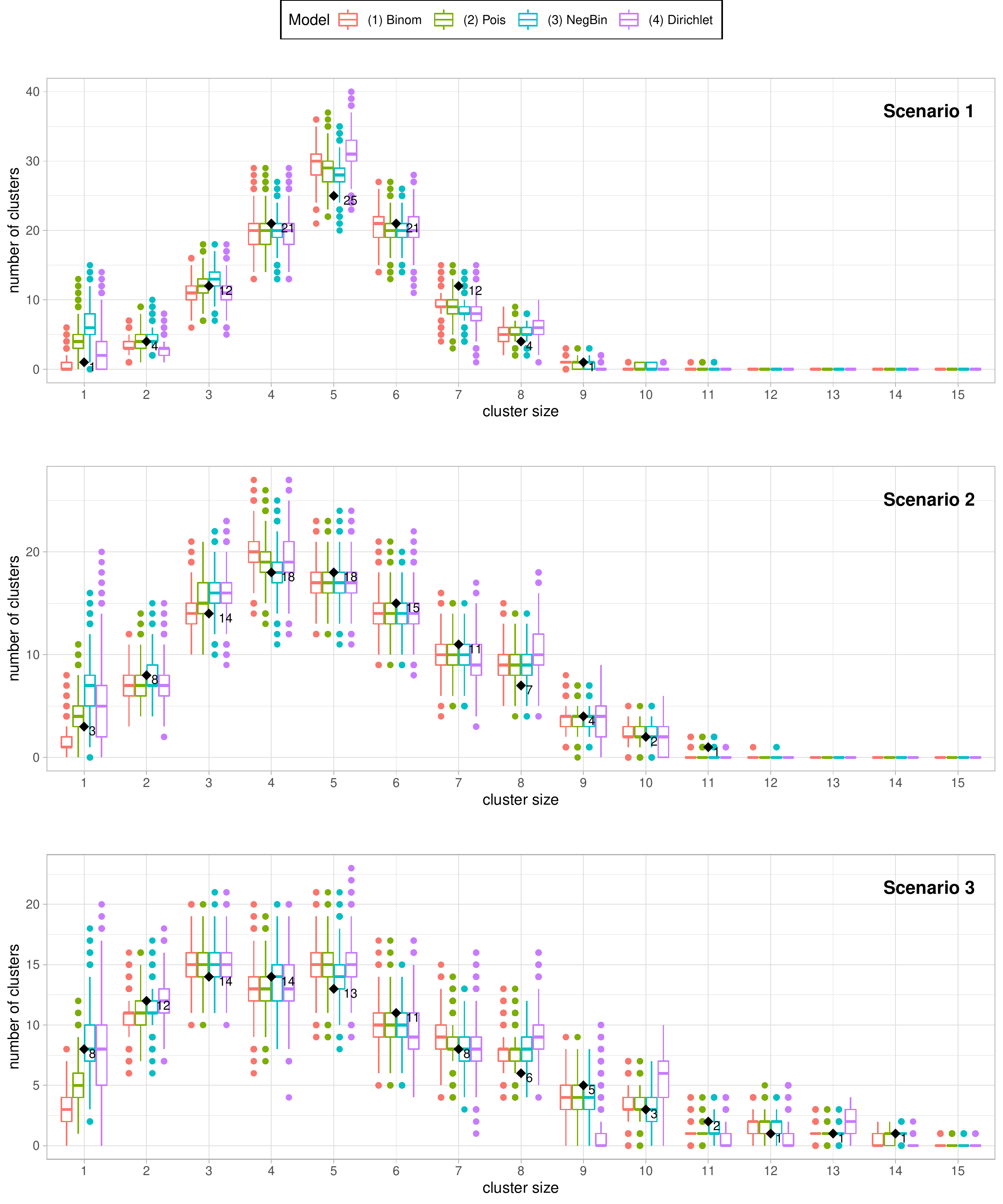}
    \caption{Posterior boxplots of the number of clusters of size $i$ for each models, $m_i$. (Top) Scenario 1 with $\beta_{true}=0.05$. (Middle) Scenario 2 with $\beta_{true}=0.05$. (Bottom) Scenario 3 with $\beta_{true}=0.05$. True $m_i$ are shown as black diamonds and annotated by their values.}
    \label{fig:figsim123}
\end{figure}


\end{document}